\pdfoutput=1
\documentclass[11pt,a4paper]{article}
\usepackage[utf8]{inputenc}
\usepackage{amsfonts, amsmath, amssymb, color, a4wide}
\usepackage{bm}
\usepackage[utf8]{inputenc} % allow utf-8 input
\usepackage[T1]{fontenc}    % use 8-bit T1 fonts
\usepackage{url}            % simple URL typesetting
\usepackage{booktabs}       % professional-quality tables
\usepackage{amsfonts}       % blackboard math symbols
\usepackage{nicefrac}       % compact symbols for 1/2, etc.
\usepackage{microtype}      % microtypography
\usepackage{fancybox,framed}
\usepackage{lmodern}
\usepackage{authblk}
\usepackage{amsfonts}
\usepackage{amsmath}
\usepackage{amssymb}
\usepackage{bm}
\usepackage{algorithm}
\usepackage{algorithmic}
\usepackage{mathtools}
\usepackage{xspace}
\usepackage{xcolor}
\usepackage{amsthm}
\usepackage{cases}
\usepackage{dsfont}
\usepackage{bbm}
\RequirePackage[shortlabels]{enumitem}
\usepackage{accents}
\usepackage{stmaryrd}

\usepackage{caption}
\usepackage{subcaption}

\newtheorem{theorem}{Theorem}[section]

\newtheorem{lemma}[theorem]{Lemma}

\newtheorem*{remark}{Remarks}

\newtheorem{assumption}{Assumption}
\newtheorem{remarks}{Remark}

\newcommand\R{\mathbb{R}}

\newcommand{\vertiii}[1]{{\left\vert\kern-0.25ex\left\vert\kern-0.25ex\left\vert #1 
		\right\vert\kern-0.25ex\right\vert\kern-0.25ex\right\vert}}

\newcommand{\abs}[2][a]{
	\IfEqCase{#1}{%
		{a}{\left\vert#2\right\rvert}%
		{0}{\vert#2\rvert}%
		{1}{\big\vert#2\big\rvert}%
		{2}{\Big\vert#2\Big\rvert}%
		{3}{\bigg\vert#2\bigg\rvert}%
		{4}{\Bigg\vert#2\Bigg\rvert}%
	}[\PackageError{abs}{Undefined option to abs: #1}{}]%
}

\usepackage[square,sort,comma,numbers]{natbib}

\makeatletter
\newenvironment{protocol}[1][htb]{%
	\renewcommand{\ALG@name}{Protocol}% Update algorithm name
	\begin{algorithm}[#1]%
	}{\end{algorithm}}
\makeatother

\newcommand{\intr}[1]{\llbracket #1 \rrbracket}

\def\argmin{\mathop{\rm Arg\,Min}}
\def\argmax{\mathop{\rm Arg\,Max}}

\newcommand{\norm}[1]{\left\lVert#1\right\rVert}

\newcommand{\Var}{\mathrm{Var}}

\title{Covariance Adaptive Best Arm Identification}
\date{}

\author[1]{El Mehdi Saad}
\author[2]{Gilles Blanchard}
\author[3]{Nicolas Verzelen}

\affil[1]{Laboratoire des Signaux et Systèmes, CentraleSupéléc, Université Paris-Saclay, France.}
\affil[3]{INRAE, Mistea, Institut Agro, Univ Montpellier, Montpellier, France.}
\affil[2]{Laboratoire de mathématiques d'Orsay, Université Paris-Saclay, France.}

\begin{document}
	\maketitle

\begin{abstract} 
		
		We consider the problem of best arm identification in the multi-armed bandit model, under fixed confidence. Given a confidence input $\delta$, the goal is to identify the arm with the highest mean reward with a probability of at least $1 - \delta$, while minimizing the number of arm pulls. While the literature provides solutions to this problem under the assumption of independent arms distributions, we propose a more flexible scenario where arms can be dependent and rewards can be sampled simultaneously. This framework allows the learner to estimate the covariance among the arms distributions, enabling a more efficient identification of the best arm. The relaxed setting we propose is relevant in various applications, such as clinical trials, where similarities between patients or drugs suggest underlying correlations in the outcomes. We introduce new algorithms that adapt to the unknown covariance of the arms and demonstrate through theoretical guarantees that substantial improvement can be achieved over the standard setting. Additionally, we provide new lower bounds for the relaxed setting and present numerical simulations that support their theoretical findings.
		%We consider the problem of best arm identification in the multi-armed bandit model, under fixed confidence. Given a confidence input $\delta$, the goal is to identify the arm with the highest mean reward with a probability of at least $1-\delta$, while minimizing the number of arm pulls. While the literature provides solutions to this problem under the assumption of independent arms distributions, we propose a more flexible scenario where arms can be dependent and rewards can be sampled simultaneously. This framework allows the learner to estimate the covariance among the arms distributions, enabling a more effitient identification of the best arm. The relaxed setting we propose is relevant in various applications, such as clinaical trials, where similiraties between patients or drugs suggest underlying correlations in the outcomes. We introduce new algorithms that adapt to the unknown covariance of the arms, and demonstrate through theoretical guarantees that substantial improvement can be achieved over the standard setting. Additionally, we provide new lower bounds for the relaxed setting, and present numerical simulations that support their theoretical findings.  	
	\end{abstract}
	
	%\vspace{-0.3cm}
	\section{Introduction and setting}
	%\vspace{-0.2cm}
	Best arm identification (BAI) is a sequential learning and decision problem that refers to finding the arm with the largest mean (average reward) among a finite number of arms in a stochastic multi-armed bandit (MAB) setting. An MAB model $\nu$ is a set of $K$ distributions in $\mathbb{R}$: $\nu_1, \dots, \nu_K$, with means $\mu_1, \dots, \mu_K$. An "arm" is identified with the corresponding distribution index. The observation consists in sequential draws ("queries" or "arm pulls") from these distributions, and each such outcome is a "reward".
	%associated respectively with arms $1, \dots, K$. 
	The learner's goal is to identify the optimal arm $i^* := \text{ArgMax}_{i \in \intr{K}} \mu_i$, efficiently. There are two main variants of BAI problems: The \textit{fixed budget setting} \cite{audibert2010best, carpentier2016tight}, where given a fixed number of queries $T$, the learner allocates queries to candidates arms and provides a guess for the optimal arm. The theoretical guarantee in this case takes the form of an upper bound on the probability $p_T$ of selecting a sub-optimal arm.  The second variant is the \textit{fixed confidence} setting \cite{garivier2016optimal, kaufmann2016complexity}, where a confidence parameter $\delta \in (0,1)$ is given as an input to the learner, and the objective is to output an arm $\psi \in \intr{K}$, such that $\mathbb{P}(\psi = i^*) \ge 1-\delta$, using the least number of arm pulls .
	%(queries to samples from arms distributions). 
	The complexity in this case corresponds to the total number of queries made before the algorithm stops and gives a guess for the best arm that is valid with probability at least $1-\delta$ according to a specified stopping rule. In this paper we specifically focus on the fixed confidence setting.

	The problem of best-arm identification with fixed confidence was extensively studied and is well understood in the literature \cite{jamieson2014lil, kalyanakrishnan2012pac, kaufmann2016complexity, garivier2016optimal}. However, in these previous works, the problem was considered under the assumption that all observed rewards are independent. More precisely, in each round $t$, a fresh sample (reward vector), independent of the past, $(X_{1,t}, \dots, X_{K,t})$ is secretly drawn from $\nu$ by the environment, and the learner is only allowed to choose one arm $A_t$ out of $K$ (and observe its reward $X_{A_t,t}$). We relax this setting by allowing simultaneous queries. Specifically, we consider the MAB model $\nu$ as a \textit{joint probability distribution} of the $K$ arms, and in each round $t$ the learner chooses a subset $C_t \subset \intr{K}$ and observes the rewards $(X_{i,t})_{i \in C_t}$ (see the Game Protocol~\ref{algo:gp}). The high-level idea of our work is that allowing multiple queries per round opens up opportunities to estimate and leverage the underlying structure of the arms distribution, which would otherwise remain inaccessible with one-point feedback. This includes estimating the covariance between rewards at the same time point. It is important to note that our proposed algorithms do not require any prior knowledge of the covariance between arms.

	Throughout this paper, we consider two cases: bounded rewards (in Section~\ref{sec:Bounded}) and Gaussian rewards (in Section~\ref{sec:Gauss}), with the following assumptions:
	\begin{assumption}{1}{ }\label{a:0}
		Suppose that:
		\begin{itemize}[noitemsep]
			\item IID assumption with respect to $t$: $(X_{t})_{t \ge 1} = (X_{1,t},\dots, X_{K,t})_{t\ge1}$ are independent and identically distributed variables following $\nu$.
			\item There is only one optimal arm: $\lvert \argmax_{i \in \intr{K}} \mu_i \rvert = 1$. 
		\end{itemize}
	\end{assumption}
	
	\begin{assumption}{2}{-B}\label{a:bounded}
		Bounded rewards: The support of $\nu$ is in $\left[0,1\right]^K$.
	\end{assumption}
	\begin{assumption}{2}{-G}\label{a:gaussian}
		Gaussian rewards: $\nu$ is a multivariate normal distribution.
	\end{assumption}

	\begin{protocol}
		%\centering
		\caption{ The Game Protocol}
		\begin{algorithmic}\label{algo:gp}
			\STATE \textbf{Parameters}:  $\delta$.
			%\STATE \qquad $p$, the number of experts allowed for prediction;
			%\STATE \qquad $m$, the number of experts allowed for observation as feedback.
			\WHILE{\textbf{[condition]}}
			\STATE Choose a subset $C \subseteq \intr{K}$.
			\STATE The environment reveals the rewards $\left(X_i\right)_{i\in  C}$.
			\ENDWHILE
			\STATE Output the selected arm: $\psi$.
		\end{algorithmic}
	\end{protocol}
	
	A round corresponds to an iteration in Protocol~\ref{algo:gp}. Denote by $i^* \in \intr{K}$ the optimal arm. The learner algorithm consists of: a sequence of queried subsets $(C_t)_t$ of $\intr{K}$, such that subset $C_t$ at round $t$ is chosen based on past observations, a halting condition to stop sampling (i.e. a stopping time written $\tau$) and an arm $\psi$ to output after stopping the sampling procedure. 
	%Hence the player's strategy consists of a triplet  $\pi = ((C_t), \tau, \psi)$ where: The \textit{sampling rule} determines, based on past observations, which subset of arms is queried at round $t$. Let $(\mathcal{F}_t)$ denote the natural filtration associated to the chosen arms and their observed rewards prior ro $t$: $\mathcal{F}_t = \sigma(C_1, (X_{i,1})_{i \in C_1}, \dots, C_t, (X_{i,t})_{i \in C_t})$. The \textit{stopping rule} $\tau$, which indicates when the player is confident to output a recommendation for the best arm. Formally, it is a stopping time with respect to the filtration $\mathcal{F}$. The \textit{recommendation rule}, which is a $\mathcal{F}_{\tau}$-measurable random variable of $\intr{K}$ consisting of the player's guess of the best arm.
	The theoretical guarantees take the form of a high probability upper-bound on the total number $N$ of queries made through the game:
 %precisely
	\begin{equation}\label{eq:budget}
	N := \sum_{t=1}^{\tau} \lvert C_t \rvert.
	\end{equation}
	%When the player is constrained to pick one arm per round, as in the multi-armed bandit setting, it holds $N_{\pi} = \tau$.
	
	\vspace{-0.4cm}
	%\paragraph{Notation.}  We summarize here some of the notation used throughout this paper. For each arm $i\in \intr{K}$, let $\hat{\mu}_{i,t} := (1/t)\sum_{s=1}^{t}X_{i,t}$ and $\mathbf{\hat{\mu}}_t := (\hat{\mu}_{1,t}, \dots, \hat{\mu}_{K,t})$. For any two random variables $G \in [0,B]^K$ and $H \in [0,B]^K$, let $\hat{d}_t(G,H)$ denote the empirical $L_2$-distance computed using $t$ samples $(G_s, H_s)_{s\le t}$ and let $d(G,H)$ denote its population counterpart. We denote $a \lesssim b$, if there exists a numerical constant independent of $a$ and $b$ such that $a\le c b \log(b)$. Let $a\wedge b := \min\{a,b\}$ and $a\vee b := \max\{a,b\}$.

	\section{Motivation and main contributions}\label{sec:intro}
	\vspace{-0.2cm} The query complexity of best arm identification with independent rewards, when arms distributions are $\sigma$-sub-Gaussian is characterized by the following quantity:
	\begin{equation}\label{eq:sample_complexity}
	H(\nu) := \sum_{i \in \intr{K}\setminus \{i^*\}} \frac{\sigma^2}{(\mu_{i^*} - \mu_{i})^2}.
	\end{equation}
	
	\vspace{-0.3cm}
	Some $\delta$-PAC algorithms \cite{jamieson2014lil} guarantee a total number of queries satisfying $\tau = \tilde{\mathcal{O}}(H(\nu)\log(1/\delta))$ (where $\tilde{\mathcal{O}}(.)$ hides a logarithmic factor in the problem parameters).

	Recall that the number of queries required for comparing the means of two variables, namely $X_1\sim \mathcal{N}(\mu_1, 1)$ and $X_2 \sim \mathcal{N}(\mu_2, 1)$, within a tolerance error $\delta$, is approximately $\mathcal{O}\left(\log(1/\delta)/(\mu_1 - \mu_2)^2\right)$. An alternative way to understand the sample complexity stated in equation \eqref{eq:sample_complexity} is to view the task of identifying the optimal arm $i^*$ as synonymous with determining that the remaining arms $\intr{K}\setminus \{i^*\}$ are suboptimal. Consequently, the cost of identifying the best arm corresponds to the sum of the comparison costs between each suboptimal arm and the optimal one.
	%an optimal sampling strategy for best arm identification with $K$ arms is one that guarantees that each arm in $\intr{K}\setminus \{i^*\}$ is decided as being sub-optimal with a minimal number of queries. The last cost corresponds to the sample complexity required for comparing a sub-optimal to the optimal one $i^*$. This intuition is valid only when variables are independent or when the learner is allowed one query per round.
	
	In practical settings, the arms distributions are not independent.  This lack of independence can arise in various scenarios, such as clinical trials, where the effects of drugs on patients with similar traits or comparing drugs with similar components may exhibit underlying correlations.
	%A hidden environmental variable may also influence all outcomes at the same time point.
	These correlations provide additional information that can potentially expedite the decision-making process \cite{kahan2012improper}. 
	
	In such cases, Protocol~\ref{algo:gp} allows the player to estimate the means and the covariances of arms. This additional information naturally raises the following question: 
	\begin{center}
		\textit{Can we accelerate best arm identification by leveraging \\the (unknown) covariance between the arms?}
	\end{center}
	We give two arguments for a positive answer to the last question: When allowed simultaneous queries (more than one arm per round as in Protocol~\ref{algo:gp}), the learner can adapt to the covariance between variables. To illustrate, consider the following toy example for $2$ arms comparison: $X_1 \sim \mathcal{N}(\mu_1, 1)$ and $X_2 = X_1+Y$ where $Y \sim \epsilon\mathcal{N}(1, 1)$ for some $\epsilon >0$. BAI algorithms in one query per round framework require $\mathcal{O}((1+\epsilon)^2\log(1/\delta)/\epsilon^2)$ queries, which gives $\mathcal{O}(\log(1/\delta)/\epsilon^2)$ for small $\epsilon$. In contrast, when two queries per round are possible, the learner can perform the following test $\mathcal{H}_0: ``\mathbb{E}[X_1-X_2] >0"$ against $\mathcal{H}_1: ``\mathbb{E}[X_1-X_2] \le 0"$. Therefore using standard test algorithms adaptive to unknown variances, such as Student's $t$-test, leads to a number of queries of the order $\mathcal{O}(\Var(X_1-X_2)\log(1/\delta)/\epsilon^2) = \mathcal{O}(\log(1/\delta))$. Hence a substantial improvement in the sample complexity can be achieved when leveraging the covariance. We can also go one step further to even reduce the sample complexity in some settings. Indeed, we can establish the sub-optimality of an arm $i$ by comparing it with another sub-optimal arm $j$ faster than when comparing it to the optimal arm $i^*$. To illustrate consider the following toy example with $K=3$: let $X_1 \sim \mathcal{N}(\mu_1, 1)$, $X_2 = X_1 + Y$ where $Y\sim \epsilon \mathcal{N}(1,1)$ and $X_3 \sim \mathcal{N}(\mu_1+2\epsilon, 1)$, with $X_1,Y$ and $X_3$ independent. $X_3$ is clearly the optimal arm. Eliminating $X_1$ with a comparison with $X_3$ requires $\mathcal{O}(\log(1/\delta)/\epsilon^2)$ queries, while comparing $X_1$ with the sub-optimal arm $X_2$ requires $\mathcal{O}(\Var(X_1-X_2)\log(1/\delta)/\epsilon^2) = \mathcal{O}(\log(1/\delta))$.

	The aforementioned arguments can be adapted to accommodate bounded variables, such as when comparing correlated Bernoulli variables. These arguments suggest that by utilizing multiple queries per round (as shown in Protocol~\ref{algo:gp}), we can expedite the identification of the best arm compared to the standard one-query-per-setting approach. Specifically, for variables bounded by $1$, our algorithm ensures that the cost of eliminating a sub-optimal arm $i \in \intr{K}$ is given by:
	\begin{equation}\label{eq:cost}
	\min_{\substack{j \in \intr{K}:\ \mu_j > \mu_i }} \left\lbrace \frac{\Var(X_j-X_i)}{(\mu_j - \mu_i)^2}+\frac{1}{\mu_j - \mu_i} \right\rbrace.
	\end{equation}
	
	The additional term $1/(\mu_j-\mu_i)$ arises due to the sub-exponential tail behavior of the sum of bounded variables. To compare the quantity in \eqref{eq:cost} with its counterpart in the independent case, it is important to note that the minimum is taken over a set that includes the best arm, and that $\Var(X_{i^*} - X_i) \le 2(\Var(X_{i^*}) + \Var(X_i))$. Consequently, when the variables are bounded, the quantity \eqref{eq:cost} is no larger than a numerical constant times $H(\nu)$ (its independent case counterpart in \ref{eq:sample_complexity}), with potentially a significant improvements if there is positive correlation between an arm $j \in \intr{K}$ with a higher mean and arm $i$ (so that $\Var(X_j-X_i)$ is small). As a result, our algorithm for bounded variables has a sample complexity, up to a logarithmic factor, that corresponds to the sum of quantities \eqref{eq:cost} over all sub-optimal arms.

	We expand our analysis to encompass the scenario where arms follow a Gaussian distribution. In this context, our procedure ensures that the cost of eliminating a sub-optimal arm $i\in \intr{K}$ is given by:
	\[
	\min_{\substack{j \in \intr{K}:\\ \mu_j > \mu_i }}  \left\lbrace \frac{\Var(X_j-X_i)}{(\mu_j - \mu_i)^2}\vee 1 \right\rbrace.
	\] 
	Similar to the setting with bounded variables, the aforementioned quantity is always smaller than its counterpart for independent arms when all variables $X_j$ have a unit variance.
	
	In Section~\ref{sec:lower}, we present two lower bound results for the sample complexity of best arm identification in our multiple query setting, specifically when the arms are correlated. Notably, these lower bounds pinpoint how the optimal sample complexity may decrease with the covariance between arms, making them the first of their kind in the context of best arm identification, to the best of our knowledge. The presented lower bounds are sharp, up to a logarithmic factor, in the case where all arms are positively correlated with the optimal arm. However, it remains an open question to determine a sharper lower bound applicable to a more general class of distributions with arbitrary covariance matrices.
	
	In Section~\ref{sec:convex} of the appendix, we introduce a new algorithm that differs from the previous ones by performing comparisons between the candidate arm and a convex combination of the remaining arms, rather than pairwise comparisons. We provide theoretical guarantees for the resulting algorithm. 
	
	Lastly, we conduct numerical experiments using synthetic data to assess the practical relevance of our approach.

	\vspace{-0.3cm}
	\section{Related work}
	\vspace{-0.2cm}
	
	\paragraph{Best arm identification:}
	
	%The introduction of the best arm identification problem dates back to \cite{thompson1933likelihood} in the context of medical trials. It the machine learning literature, it was re-introduced by \cite{even2002pac}. In this paper, we focus on the fixed confidence setting \cite{kaufmann2016complexity, garivier2016optimal}, where the learner is given a confidence level $\delta \in (0,1)$ and should use as few queries as possible to identify the best arm. 
	
	BAI in the fixed confidence setting was studied by \cite{even2002pac}, \cite{mannor2004sample}, and \cite{even2006action}, where the objective is to find $\epsilon$-optimal arms under the PAC (``probably approximately correct") model.    
	A summary of various optimal bounds for this problem is presented in \cite{carpentier2016tight, kaufmann2016complexity}.
	Prior works on fixed confidence BAI \citep{mnih2008empirical} and \citep{jourdan2023dealing} developed strategies adaptive to the unknown variances of the arms. In contrast, our proposed algorithms demonstrate adaptability to all entries of the covariance matrix of the arms. In particular, we establish that in the worst-case scenario, where the arms are independent, our guarantees align with the guarantees provided by these previous approaches.
	\vspace{-0.2cm}
	\paragraph{Covariance in the Multi-Armed Bandits model:}

	Recently, the concept of leveraging arm dependencies in the multi-armed bandit (MAB) model for best-arm identification was explored in \cite{gupta2021best}. However, their framework heavily relies on prior knowledge, specifically upper bounds on the conditional expectation of rewards from unobserved arms given the chosen arm's reward.
	In a similar vein, a game protocol that allows simultaneous queries was examined in \cite{liu2014most}. However, their objective differs from ours as their focus is on identifying the most correlated arms, whereas our primary goal is to identify the arm with the highest mean reward.
	
	%The extension of the standard multi-armed bandit setting to multiple-point bandit feedback was also considered in the stochastic combinatorial semi-bandit problem (\citealp{audibert2014regret}, \citealp{cesa2012combinatorial}, \citealp{chen2013combinatorial} and \citealp{gai2012combinatorial}). At each round $t \ge 1$, the learner pulls $m$ out of $K$ arms and receives the sum of the pulled arms rewards. The objective is to maximize the cumulative regret with respect to the best choice of arms. \cite{degenne2016combinatorial} proposed an algorithm adaptive to the covariance structure of the problem, but it requires an upper-bound on the covariance matrix of the arm reward distribution. A more recent work by \citep{perrault2020covariance} presented an improved version of the algorithm that does not rely on prior knowledge of the covariance matrix.
	The extension of the standard multi-armed bandit setting to multiple-point bandit feedback was also considered in the stochastic combinatorial semi-bandit problem (\citealp{audibert2014regret}, \citealp{cesa2012combinatorial}, \citealp{chen2013combinatorial} and \citealp{gai2012combinatorial}). At each round $t \ge 1$, the learner pulls $m$ out of $K$ arms and receives the sum of the pulled arms rewards. The objective is to minimize the cumulative regret with respect to the best choice of arms subset. \citep{perrault2020covariance} proposed an algorithm that adapts to the covariance of the covariates within the same arm.
	While this line of research shares the intuition of exploiting the covariance structure with our paper, there are essential differences between the two settings. In the combinatorial semi-bandit problem, the learner receives the sum of rewards from all selected arms in each round and aims to minimize cumulative regret, necessitating careful exploration during the game. In contrast, our approach does not impose any constraint on the number of queried arms per round, and our focus is purely on exploration.   

    Simultaneous queries of multiple arms was also considered in the context of graph-based bandit problems \cite{caron2012leveraging}. However, in these studies, it is assumed that the distributions of arms are independent.
	\vspace{-0.2cm}
	\paragraph{Model selection racing:}
	
	Racing algorithms for model selection refers to the problem of selecting the best model out of a finite set efficiently. The main idea consists of early elimination of poorly performing models and concentrating the selection effort on good models. This idea was seemingly first exploited in \cite{maron1993hoeffding} through Hoeffding Racing. It consists of sequentially constructing a confidence interval for the generalization error of each (non-eliminated) model. Once two intervals become disjoint, the corresponding sub-optimal model is discarded. Later \cite{mnih2008empirical} presented an adaptive stopping algorithm using confidence regions derived with empirical Bernstein concentration inequality (\citealp{audibert2007tuning, maurer2009empirical}). The resulting algorithm is adaptive to the unknown marginal variances of the models. Similarly, \cite{balsubramani2016sequential} presented a procedure centered around sequential hypothesis testing to make decisions between two possibilities. In their setup, they assume independent samples and use Bernstein concentration inequality tailored for bounded variables to adapt to the variances of the two variables under consideration.    
	
	%Hoeffding and Bernstein races evaluate the models individually (building a confidence interval for each model using only its queries). When many models are very similar, the behavior of such algorithms suffers because the near-identical ``good" models will have to run through the whole race. To circumvent this scenario, \cite{box1978statistics} and \cite{moore1994efficient} proposed eliminating near-identical models and race only representative candidates through a statistical method called \textit{Blocking}. A more formal approach was presented by $F$-Race methods \cite{birattari2002racing}, where the similarity of models is assessed through Friedman \textit{post hoc tests}.      
	
	While the idea of exploiting the possible dependence between models was shown \cite{birattari2010f,moore1994efficient} to empirically outperform methods based on individual performance monitoring, there is an apparent lack of theoretical guarantees. This work aims to develop a control on the number of sufficient queries for reliable best arm identification, while being adaptive to the unknown correlation of the candidate arms.

	\vspace{-0.3cm}
	\section{Algorithm and main theorem for bounded variables}\label{sec:Bounded}
	\vspace{-0.2cm}
	
	In this section, we focus on the scenario where the arms are bounded by $1$. Let us establish the notation we will be using throughout. % the discussion. 
	For each arm $i \in \intr{K}$ associated with the reward variable $X_i$, we denote by $\mu_i$ the mean of $X_i$. Additionally, for any pair of arms $i, j \in \intr{K}$, we denote $V_{ij}$ as the variance of the difference between $X_i$ and $X_j$, i.e., $V_{ij} := \Var(X_i-X_j)$.  We remind the reader that the quantities $V_{ij}$ are unknown to the learner.
	
	Algorithm~\ref{algo:1} is developed based on the ideas introduced in Section~\ref{sec:intro}, which involves conducting sequential tests between every pair $(i,j)\in \intr{K}\times\intr{K}$ of arms. The key element to adapt to the covariance between arms is the utilization of the empirical Bernstein's inequality \cite{maurer2009empirical} for the sequence of differences $(X_{i,t}-X_{j,t})_{t}$ for $i,j \in \intr{K}$. To that end, we introduce the following quantity:
	\[
	\hat{\Delta}_{ij}(t,\delta) := \hat{\mu}_{j,t} - \hat{\mu}_{i,t}  - \frac{3}{2}\alpha(t,\delta)\sqrt{2\widehat{V}_{ij,t}}  - 9~\alpha^2(t,\delta),
	\]
	where, $\hat{\mu}_{i,t}$ represents the empirical mean of the samples obtained from arm $i$ up to round $t$, and $\widehat{V}_{ij,t}$ denotes the empirical variance associated with the difference variable $(X_i-X_j)$ up to round $t$. The term $\alpha(t,\delta)$ is defined as $\alpha(t,\delta) := \sqrt{\log(1/\delta_t)/(t-1)}$, and $\delta_t$ is given by $\delta/(2K^2t(t+1))$.

	By leveraging the empirical Bernstein's inequality (restated in Theorem~\ref{conc:bern}), we can establish that if $\hat{\Delta}_{ij}(t,\delta)>0$ at any time $t$, then with a probability of at least $1-\delta$, we have $\mu_i < \mu_j$. This observation indicates that arm $i$ is sub-optimal. Furthermore, when $\mu_j >\mu_i$ for $i,j \in \intr{K}$, a sufficient sample sizes ensuring that the quantity $\hat{\Delta}_{ij}(t,\delta)$ is positive is proportional to:
	\begin{equation}\label{eq:lambda_ij}
	\Lambda_{ij} := \left( \frac{V_{ij}}{(\mu_i - \mu_j)^2}+\frac{1}{\mu_j - \mu_i} \right) \log\left(\frac{K}{(\mu_j - \mu_i)\delta}\right).
	\end{equation}
	Moreover, we show in Lemma~\ref{lem:ord_1} in the appendix that the last quantity is necessary, up to a smaller numerical constant factor, in order to have $\hat{\Delta}_{ij}(t,\delta)>0$.

	%	Algorithm~\ref{algo:g1} adopts a successive elimination approach, based on the tests: $\hat{\Delta}_{ij}(t,\delta)>0$. Recall that our objective is to ensure that any sub-optimal arm $i$ is queried at most $\min_j \Lambda_{ij}$ where the minimum is over arms with means larger than $\mu_i$. Achieving this goal raises the following challenge when using a successive elimination approach: the arm $j^*$ achieving the minimum of $\Lambda_{ij}$ may be eliminated early during the process, and the algorithm is constrained to compare arm $i$ with other arms with larger complexity $\Lambda_{ij}$. To deal with this limitation, it is usefull to continue querying arms even after deciding their sub-optimality through the pairwise tests. 
	%	
	%	We introduce in Algorithm~\ref{algo:1} at round $t$ two sets: $S_t$ of arms that are candidate to be optimal, initialized by $S_1=\intr{K}$. Let $C_t$ denote the set of arms that are queried at round $t$. Naturally, $C_t$ contains $S_t$, additionally it contains freshly eleminated arms from $S_t$, that we hope will help eliminate candidate arms from $S_t$ faster. The last precision to be made is about how much longer the algorithm should keep sampling an arm that was eliminated from $S_t$. We prove in Theorem~\ref{lem:elim_1}, that when arm $j$ is eliminated at round $t$, it is sufficient to keep sampling it up to round $19~t$ (we discuss the constant $19$ in Remark~\ref{rem:1}).     
	
	Algorithm~\ref{algo:1} follows a successive elimination approach based on the tests $\hat{\Delta}_{ij}(t,\delta)>0$. Our objective is to ensure that any sub-optimal arm $i$ is queried at most $\min_j \Lambda_{ij}$ times, where the minimum is taken over arms with means larger than $\mu_i$. However, this approach poses a challenge: the arm $j^*$ that achieves the minimum of $\Lambda_{ij}$ may be eliminated early in the process, and the algorithm is then constrained to compare arm $i$ with other arms costing larger complexity $\Lambda_{ij}$. To address this limitation, it is useful to continue querying arms even after deciding their sub-optimality through the pairwise tests.
	
	In Algorithm~\ref{algo:1}, we introduce two sets at round $t$: $S_t$, which contains arms that are candidates to be optimal, initialized as $S_1=\intr{K}$, and $C_t$, which represents the set of arms queried at round $t$. Naturally, $C_t$ contains $S_t$ and also includes arms that were freshly eliminated from $S_t$, as we hope that these arms will help in further eliminating candidate arms from $S_t$ more quickly.

	An important consideration is how long the algorithm should continue sampling an arm that has been eliminated from $S_t$. In Theorem~\ref{lem:elim_1}, we prove that when arm $j$ is eliminated at round $t$, it is sufficient to keep sampling it up to round $82t$ (the constant $82$ is discussed in Remark~\ref{rem:1}). The rationale behind this number of additional queries is explained in the sketch of the proof of Theorem~\ref{thm:2}. It suggests that if arm $j$ fails to eliminate another arm after round $82t$, then the arm that eliminated $j$ from $S_t$ can ensure faster eliminations for the remaining arms in $S_t$.

	\begin{algorithm} 
		%\centering
		\caption{Pairwise-BAI \label{algo:1}}
		\begin{algorithmic}[1]
			\STATE \textbf{Input}  $\delta$.
			\STATE \textbf{Initialization:}
			\STATE \hspace{1cm} Query all arms for $2$ rounds and compute the empirical means vector $\hat{\bm{\mu}}_t$, $t \gets 3$. 
			\STATE \hspace{1cm} $S \gets \intr{K}$, \qquad /*\texttt{Set of candidate arms}*/
			\STATE \hspace{1cm} $C \gets \intr{K}$, \qquad /*\texttt{Set of queried arms}*/
			%\STATE \hspace{1cm} $\mathbf{\hat{\mu}}_0 \gets (0, \dots, 0)$, $t \gets 1$.
			\WHILE{ $\lvert S\rvert>1$ } 
			\STATE Jointly query all arms in $ C$.
			\STATE Update $\hat{\bm{\mu}}_t$ and compute $\max_{j \in C} \hat{\Delta}_{ij}( t, \delta) $ for each $i \in S$.
			\FOR{$i \in S$}
			\IF{$\max_{j \in C} \hat{\Delta}_{ij}( t, \delta) >0$}
			\STATE 	Eliminate $i$ from $S$: $S \gets S \setminus \{i\}$. \qquad/* \texttt{$i$ is  sub-optimal} */
			\STATE Mark $i$ for elimination from $C$ at round: $82~t$. \label{line:algo}
			\ENDIF
			\ENDFOR
			\STATE {\bf Increment} t.
			\ENDWHILE
			\STATE  {\bf Return} $S$. 
		\end{algorithmic}
	\end{algorithm}

	\begin{theorem}\label{thm:2}
		Suppose Assumption~\ref{a:0} and~\ref{a:bounded} holds. Consider Algorithm~\ref{algo:1}, with input $\delta \in (0,1)$. We have with probability at least $1-\delta$: the algorithm identifies the best arm and the total number of queries $N$ satisfies:
		\begin{equation}\label{eq:b1}
		N \le c \log(K\Lambda\delta^{-1}) \sum_{i \in \intr{K}\setminus \{i^*\}} ~ \min_{\substack{j \in \intr{K}:\\ \mu_j > \mu_i }}~  \left\lbrace\frac{V_{ij}}{(\mu_i - \mu_j)^2}+ \frac{1}{\mu_j - \mu_i} \right\rbrace,
		\end{equation}
		where $\Lambda = \max_{i \in \intr{K}} \min_{\substack{j \in \intr{K}:\\ \mu_j > \mu_i }}~  \left\lbrace\frac{V_{ij}}{(\mu_i - \mu_j)^2}+ \frac{1}{\mu_j - \mu_i} \right\rbrace$ and $c$ is a numerical constant.
		
		Moreover, if we omit line~\ref{line:algo} from Algorithm~\ref{algo:1}, that is we do not query non-candidate arms, we have with probability at least $1-\delta$: the algorithm identifies the best arm and the total number of queries $N$ satisfies:
		\begin{equation}\label{eq:b2}
		N \le c \log(K\Lambda\delta^{-1}) \sum_{i \in \intr{K}\setminus \{i^*\}} ~   \left\lbrace\frac{V_{ii^*}}{(\mu_{i^*} - \mu_i)^2}+ \frac{1}{\mu_{i^*} - \mu_i} \right\rbrace,
		\end{equation}
		where $\Lambda$ is defined above and $c$ is a numerical constant.
	\end{theorem}
	
	\paragraph{Summary of the proof for bound \eqref{eq:b1}:} 
	Let $i$ denote a sub-optimal arm and $\Upsilon_i := \argmin_{j} \Lambda_{ij}$. First, we show that at round $t$, if $i \in S_t$ then necessarily we have $\Upsilon_i \cap C_t \neq \emptyset$. We proceed by a contradiction argument: assume $\Upsilon_i \cap C_t = \emptyset$, let $j$ denote the element of $\Upsilon_i$ with the largest mean. Let $k$ denote the arm that eliminated $j$ from $S$ at a round $s <t$. Lemma~\ref{lem:ord_2} shows that, since $\hat{\Delta}_{jk}(s,\delta) >0$, we necessarily have $\log(1/\delta)\Lambda_{jk}/4 \lesssim s$. Moreover, $j$ was kept up to round $82 s$ and in this last round we had $\hat{\Delta}_{ij}(82s,\delta) \le 0$, which gives by Lemma~\ref{lem:ord_2}: $82s \lesssim (25/2)\log(1/\delta_{82s})\Lambda_{ij}$. Combining the two bounds on $s$, gives: $\Lambda_{jk}<\Lambda_{ij}$. Finally, we use the ultra-metric property satisfied by the quantities $\Lambda_{uv}$ (Lemma~\ref{lem:ultram}), stating that: $\Lambda_{ik} \le \max\{\Lambda_{ij}, \Lambda_{jk}\}$. Combining the last inequality with the latter, we get $\Lambda_{ik}\le\Lambda_{ij}$, which means that $k \in \Upsilon_i$. The contradiction arises from the fact that $j$ is the element of the largest mean in $\Upsilon_i$ and arm $k$ eliminated $j$ (hence $\mu_k> \mu_j$). We conclude that necessarily $\Upsilon_i \cap C_t \neq \emptyset$, therefore by Lemma~\ref{lem:ord_2}, at a round of the order of $\min_{j} \Lambda_{ij}$, we will necessarily have $\hat{\Delta}_{ij}(t,\delta)>0$ for some $j\in C_t$.     
	\begin{remarks}\label{rem:1}
		From a practical standpoint, to keep sampling an arm eliminated at $t$ for additional $81~t$ rounds is a conservative approach. The stated value of $81$ is determined by specifics of the proof, but we believe that it can be optimized to a smaller constant based on the insights gained from numerical simulations. Furthermore, even if we omit the oversampling step,
		%the last instruction 
		as presented in Theorem~\ref{thm:2}, the algorithm
		is still guaranteed to identify the best arm with probability
		$1-\delta$. Only the query complexity guarantee is weaker,  
		%we still have weaker guarantees that can be useful in practice. 
		but the algorithm may still %perform well and 
		lead to effective arm elimination, 
		although with potentially slightly slower convergence.
		%compared to the case where the additional queries are included.
	\end{remarks}
	\begin{remarks}\label{rem:2}
		The idea of successive elimination based on evaluating the differences between variables was previously introduced in \cite{saad2021fast}, in the context of model selection aggregation. Their analysis allows having a bound slightly looser than bound \eqref{eq:b2} (with the distances between variables instead of variances). On the other hand, Theorem~\ref{thm:2} in our paper provides a sharper bound in \eqref{eq:b1}.
	\end{remarks}
	First, it is important to note that bound \eqref{eq:b1} is always sharper (up to a numerical constant factor) than bound \eqref{eq:b2} because the optimal arm is included in the set over which the minimum in bound \eqref{eq:b1} is taken. On the other hand, bound \eqref{eq:b1} can be smaller than bound \eqref{eq:b2} by a factor of $1/K$. This situation can arise in scenarios where the $K-1$ sub-optimal arms, which have close means, are highly correlated, while the optimal arm is independent of the rest. In such a situation, the terms in the sum in \eqref{eq:b1} corresponding to the correlated sub-optimal arms are relatively small, except for the arm (denoted as $i$) with the largest mean among the $K-1$ correlated suboptimal rewards. That last remaining arm will be eliminated by the optimal arm, incurring a potentially large cost, but for that arm only. %(due to the independence between arms $i$ and $i^*$). 
	On the other hand, each term in the second bound \eqref{eq:b2} would be of the order of the last cost.

	To provide perspective on our guarantees compared to those developed in the independent one query per round setting, it is worth noting that the standard guarantees in that setting provide a sample complexity corresponding to $\sum_{i \neq i^*}\log(1/\delta)/ (\mu_{i^*}-\mu_i)^2$. Since variables are bounded by $1$, their variances are bounded by $1$. Therefore, in the worst case, we recover the previous guarantees with a numerical constant factor of $2$.
	
	A refined adaptive algorithm presented in \cite{mnih2008empirical} also utilizes a successive elimination approach using confidence intervals for the arm means based on the empirical Bernstein's inequality. However, unlike our algorithm, they use the concentration inequality to evaluate each arm's mean independently of the other arms. Their approach allows for adaptability to individual arm variances, and is particularly beneficial are small, i.e., $\Var(X_i) \ll 1$. The sample complexity of their algorithm is of the order of:
	\[
	\log(K \Gamma \delta^{-1}) \sum_{i\neq i^*} \frac{\Var(X_i)+\Var(X_{i^*})}{(\mu_{i^*}-\mu_{i^*})^2}\ ,
	\]
	where $\Gamma := \max_{i\neq i^*} (\Var(X_i)+\Var(X_{i^*})/(\mu_{i^*}-\mu_{i^*})^2$. Neglecting numerical constant factors, the last bound is larger than both our bounds \eqref{eq:b1} and \eqref{eq:b2}. This is because the variance of the differences can be bounded as follows: $V_{ii^*} \le 2(\Var(X_i)+\Var(X_{i^*}))$.

	\vspace{-0.3cm}
	\section{Algorithms and main theorem for Gaussian distributions}\label{sec:Gauss}   
	\vspace{-0.2cm}
	
	In this section, we address the scenario where arms are assumed to follow a Gaussian distribution. We consider a setting where the learner has no prior knowledge about the arms' distribution parameters, and we continue using the notation introduced in the previous section. The main difference between this case and the bounded variables setting, other than the form of the sample complexity obtained, is the behavior of the algorithm when variances of differences between variables tend to $0$, displayed by the second bound in Theorem~\ref{thm:0}.

     Our algorithm relies on the empirical Bernstein inequality, which was originally designed in the literature for bounded variables \cite{audibert2007tuning, maurer2009empirical}. We have extended this inequality to accommodate Gaussian variables by leveraging existing Gaussian concentration results. Note that extending such inequalities more generally for sub-Gaussian variables is a non-trivial task. One possible direction 
    %to extend the considered class of distributions 
    is to suppose that arms follow a sub-Gaussian distribution and satisfy a Bernstein moment assumption (such extensions were pointed by works on bounded variables e.g., \citealp{balsubramani2016sequential}). Given the last class of distributions, we can build on the standard Bernstein inequality with known variance, then plug in an estimate of the empirical variance leveraging the concentration of quadratic forms (see \citealp{bellec2019concentration}). However, it remains uncertain whether an extension for sub-Gaussian variables (without additional assumptions) is practically feasible.
 
	We extend the previous algorithm to the Gaussian case by performing sequential tests between pairs of arms $(i,j)\in \intr{K}$. We establish a confidence bound for the difference variables $(X_i - X_j)$ (Lemma~\ref{lem:conc_g} in the appendix) and introduce the following quantity:
	\begin{equation}\label{def:delta_p}
	   \hat{\Delta}'_{ij}(t,\delta) := \hat{\mu}_{j,t}-\hat{\mu}_{i,t} -\frac{3}{2}\alpha(t,\delta)\sqrt{2f(\alpha(t,\delta))~\hat{V}_{ij,t}},   
	\end{equation}
	 
	where the $f$ is defined by: $f(x) = \exp(2x+1)$ if $x\ge 1/3$, and $f(x) = 1/(1-2x)$ otherwise.
	
	Using the empirical confidence bounds on the differences between arms, we apply the same procedure as in the case of bounded variables. We make one modification to Algorithm~\ref{algo:1}: we use the quantities $\hat{\Delta}'_{ij}(t, \delta)$ instead of $\hat{\Delta}_{ij}(t,\delta)$.
	
	By following the same analysis as in the bounded setting, we establish that the sample complexity required for the comparison tests between arms $i$ and $j$ is characterized by the quantity $\frac{V_{ij}}{(\mu_i-\mu_j)^2} \vee 1$. 
	The following theorem provides guarantees on the algorithm presented above:
	\begin{theorem}\label{thm:0}
		Suppose Assumptions~\ref{a:0} and~\ref{a:gaussian} holds. Consider the algorithm described above, with input $\delta \in (0,1)$. We have with probability at least $1-\delta$: the algorithm identifies the best arm and the total number of queries $N$ satisfies:
		\begin{equation}\label{eq:g1}
		N \le c \log(K\Lambda\delta^{-1}) \sum_{i \in \intr{K}\setminus \{i^*\}} ~ \min_{\substack{j \in \intr{K}:\\ \mu_j > \mu_i }}~  \left\lbrace\frac{V_{ij}}{(\mu_i - \mu_j)^2} \vee 1 \right\rbrace,	
		\end{equation}
		where $\Lambda = \max_{i \in \intr{K}} \min_{\substack{j \in \intr{K}:\\ \mu_j > \mu_i }}~  \left\lbrace\frac{V_{ij}}{(\mu_i - \mu_j)^2} \vee 1 \right\rbrace$ and $c$ is a numerical constant.
		
		Moreover, if we omit line~\ref{line:algo} from Algorithm~\ref{algo:1}, that is we do not query non-candidate arms, we have with probability at least $1-\delta$: the algorithm identifies the best arm and the total number of queries $N$ satisfies:
		\begin{equation}\label{eq:g2}
		N \le 3K+c \log\left( \frac{K\delta^{-1}}{\log\left(1+1/\Lambda\right)} \right) \sum_{i \in \intr{K}\setminus \{i^*\}} ~   \frac{1}{\log\left(1+\frac{(\mu_{i^*}-\mu_i)^2}{V_{ii^*}}\right)},	
		\end{equation}
		where $\Lambda = \max_{i\neq i^*} V_{ii^*}/(\mu_{i^*}-\mu_i)^2$ is defined above and $c$ is a numerical constant.
	\end{theorem}
 \begin{remark}
    On the sample complexity cost of being adaptive to the variance: If the sample size is larger than  $\log(1/\delta)$, the cost of plugging in the empirical variance estimate into the Chernoff's concentration inequality is only a multiplicative constant slightly larger than one (nearly $1+2\sqrt{\log(1/\delta)/n}$). However, in the case of a small sample regime ($n<\log(1/\delta)$), the cost is a multiplicative factor of $\exp(\sqrt{\log(1/\delta)/n}+1/2)$ due to the nature of the left tail of the chi-squared distribution (see Sections D and K of the appendix for detailed calculations). For most natural regimes, the number of queries made for each arm is larger than $\log(1/\delta)$, hence the last described effect does not arise. However, in some specific regimes (such as the case of very small variances of the arms) an optimal algorithm should query less than $\log(1/\delta)$ samples, which necessitates introducing the exponential multiplicative term above into the concentration upper bound. This translates into a different form of guarantee presented in Theorem~\ref{thm:0}, inequality \eqref{eq:g2}. It is important to note the cost in this regime cannot be avoided as highlighted by our lower bound presented in Theorem~\ref{thm:lower_bound_g}.  
 \end{remark}
	Theorem~\ref{thm:0} above shows that Algorithm~\ref{algo:1} is applied to Gaussian distribution guarantees that each sub-optimal arm $i$ is eliminated after roughly $\min_j V^2_{ij}/(\mu_i-\mu_j)^2$ queries. Bounds \eqref{eq:g1} and \eqref{eq:g2} derived from our algorithm are smaller than the standard complexity bound in the independent case, which is given by $\sum_{i\neq i^*} \sigma^2/(\mu_i - \mu_{i^*})^2$ when all arms have variances smaller than $\sigma^2$. It is important to note that unlike most existing procedures in the literature that achieve the standard complexity bound, our algorithm does not require knowledge of the parameter $\sigma^2$.

	\paragraph{About the upper bound \eqref{eq:g2}:} This bound can be further bounded by $\log(\delta^{-1})\sum_{i \neq i^*}\frac{V_{ii^*}}{(\mu_i-\mu_{i^*})^2}\vee 1$. The form presented in Theorem~\ref{thm:0} is particularly sharp when the variances $V_{ii^*}$ tend to zero, resulting in a constant upper bound. Recently, in the independent setting where variances are unknown, \cite{jourdan2023dealing} analyzed the comparison of two arms $i$ and $j$. They derived a complexity of the order $1/\log\left(1+\frac{(\mu_i - \mu_j)^2}{(\sigma_i^2+\sigma_j^2)}\right)$, where $\sigma_i^2$ and $\sigma_j^2$ are the variances of arms $i$ and $j$, respectively. This result is reflected in our bound \eqref{eq:g2}, where instead of the sum of variances, our bound considers the variance of the difference, leading to a sharper bound. In Section~\ref{sec:lower}, we provide a lower bound that nearly matches \eqref{eq:g2}.

	\vspace{-0.3cm}
	\section{Lower bounds}\label{sec:lower}
	\vspace{-0.2cm}
	
	In this section, we present lower bounds for the problem of best arm identification with multiple queries per round, following Protocol~\ref{algo:gp}. We provide lower bounds for both the bounded distributions setting (Theorem~\ref{thm:lower_bound_b}) and the Gaussian distributions setting (Theorem~\ref{thm:lower_bound_g}). It is important to note that our lower bounds are derived considering a class of correlated arm distributions. Therefore, the results obtained in the standard one-query-per-round setting \cite{kaufmann2016complexity} do not hold in our setting.

	Our first lower bound is derived for the case where arms follow a Bernoulli distribution. Let $\bm{\mu} = (\mu_i)_{i\in \intr{K}}$ be a sequence of means in $[1/4,3/4]$, denote by $i^*$ the index of the largest mean. Consider a sequence of positive numbers $\bm{V} = (V_{ii^*})_{i \in \intr{K}}$. We define $\mathbb{B}_K(\bm{\mu}, \bm{V})$ as the set of Bernoulli arm distributions such that: (i) $\mathbb{E}[X_i]=\mu_i$; (ii) $\Var(X_i - X_{i^*}) \le V_{ii^*}$ for all $i \in \intr{K}$. 
	
	For Bernoulli distributions, the variances and means are linked. Specifically, Lemma~\ref{lem:ber} demonstrates that for any pair of Bernoulli random variables $(B_1, B_2)$ with means $(b_1,b_2)$, the following inequality holds: $(b_1-b_2)-(b_1-b_2)^2 \le \Var(B_1-B_2) \le \min\{2-(b_1+b_2); b_1+b_2\} - (b_1-b_2)^2$.
	
	We assume that the sequences $\bm{V}$ and $\bm{\mu}$ satisfy the aforementioned condition for $(X_i,X_{i^*})_{i \in \intr{K}}$, as otherwise the class $\mathbb{B}_{K}(\bm{\mu}, \bm{V})$ would be an empty set. Theorem~\ref{thm:lower_bound_b} provides a lower bound on the sample complexity required for best arm identification in the worst-case scenario over the class $\mathbb{B}_K(\bm{\mu}, \bm{V})$.
	
	\begin{theorem}\label{thm:lower_bound_b}
		Let $K \ge 2$ and $\delta \in (0,1)$. For any $\delta$-sound algorithm, we have
		\[
		\max_{\mathcal{B} \in \mathbb{B}_K(\bm{\mu}, \bm{V})} \mathbb{E}_{\mathcal{B}} \left[N\right] \ge \frac{1}{8}\log(1/4\delta) \sum_{i \in \intr{K}\setminus \{i^*\}} \max\left\lbrace \frac{V_{ii^*}}{(\mu_{i^*}-\mu_i)^2}, \frac{1}{\mu_{i^*}-\mu_i}\right\rbrace,
		\] 
		where $N$ is the total number of queries.
	\end{theorem}

	The presented lower bound takes into account the correlation between arms by incorporating the quantities $V_{ii^*}$ as upper bounds for the variances of the differences between arm $i$ and the optimal arm $i^*$. This lower bound indicates that for class $\mathbb{B}_K(\bm{\mu}, \bm{V})$, Algorithm~\ref{algo:1} is nearly optimal.
	
	Next, we provide a lower bound in the Gaussian case. Let $\bm{\mu} = (\mu_i)_{i\in \intr{K}}$ be a sequence of means, where $i^*$ denotes the index of the largest mean. We also consider a sequence of positive numbers $\bm{V} = (V_{ii^*})_{i \in \intr{K}}$. We define $\mathbb{G}_K(\bm{\mu}, \bm{V})$ as the set of Gaussian arm distributions satisfying the following conditions: (i) $\mathbb{E}[X_i] = \mu_i$, (ii) $\Var(X_i - X_{i^*}) = V_{ii^*}$, and (iii) $\Var(X_i) \ge 1$ for all $i\in \intr{K}$. Theorem~\ref{thm:lower_bound_g} provides a lower bound on the sample complexity required for best arm identification in the worst-case scenario over the class $\mathbb{G}_K(\bm{\mu}, \bm{V})$.

	%Theorem below
	
	\begin{theorem}\label{thm:lower_bound_g}
		Let $K \ge 2$ and $\delta \in (0,1)$. For any $\delta$-sound algorithm, we have
		\[
		\max_{\mathcal{G} \in \mathbb{G}_K(\bm{\mu}, \bm{V})} \mathbb{E}_{\mathcal{G}} \left[N\right] \ge 2\log(1/4\delta) \sum_{i \in \intr{K}\setminus \{i^*\}} \frac{1}{\log\left(1+\frac{(\mu_i-\mu_{i^*})^2}{V_{ii^*}^2}\right)},
		\] 
		where $N$ is the total number of queries.
	\end{theorem}
	Theorem~\ref{thm:lower_bound_g} demonstrates that our algorithm achieves near-optimal performance over $\mathbb{G}_K(\bm{\mu}, \bm{V})$, up to a logarithmic factor.

	\vspace{-0.3cm}
	\section{Numerical simulations}
	\vspace{-0.2cm}
	
	We consider the Gaussian rewards scenario. We compare our algorithm Pairwise-BAI (Algorithm~\ref{algo:1}) to 3 benchmark algorithms: Hoeffding race \cite{maron1993hoeffding}, adapted to the Gaussian setting (consisting of successive elimination based on Chernoff's bounds) and LUCB \cite{kalyanakrishnan2012pac}, which is an instantiation of the upper confidence bound (UCB) method. We assume that the last two algorithms have a prior knowledge on the variances of the arms. The third benchmark algorithm consists of using a successive elimination approach using the empirical estimates of the variances. We evaluated two variations of our algorithm. The first one, Pairwise-BAI+, implemented Algorithm~\ref{algo:1} for Gaussian variables. In this instance, we modified line~\ref{line:algo} by continuing to sample sub-optimal arms that were eliminated at round $t$ until round $2t$ instead of $82t$. We stress that
	both variants guarantee a $\delta$-sound
	decision on the optimal arm
	(see Theorem~\ref{thm:0}).
	The second instance involved removing the last instruction, meaning we directly stopped querying sub-optimal arms. Figure~\ref{fig:1} displays the average sample complexities for each considered algorithm. As expected, the larger the correlation between arms, the better Pairwise-BAI performs.
	
	\begin{figure}
		\centering
		\begin{subfigure}{.5\textwidth}
			\centering
			\includegraphics[width=\linewidth]{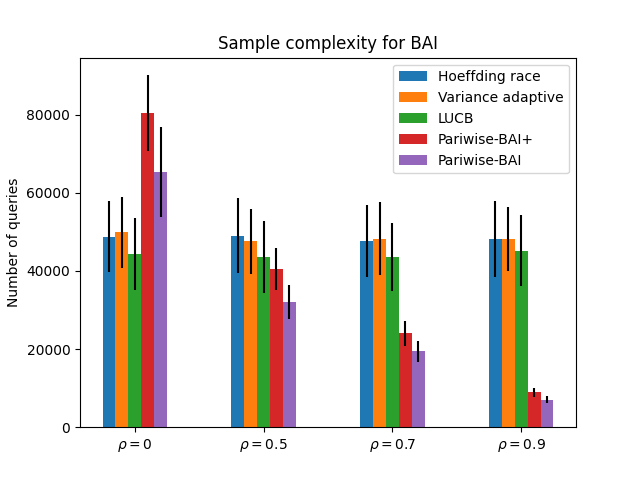}
			%\caption{A subfigure}
			%\label{fig:sub1}
		\end{subfigure}%
		\begin{subfigure}{.5\textwidth}
			\centering
			\includegraphics[width=\linewidth]{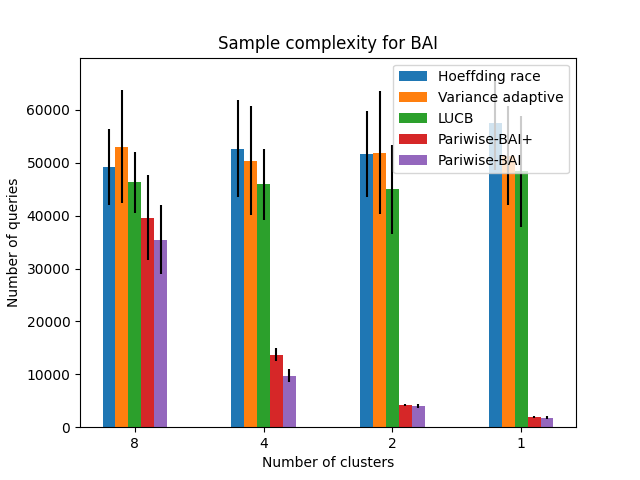}
			%\caption{A subfigure}
			%\label{fig:sub2}
		\end{subfigure}
		\caption{Left figure: the average sample complexities, in $4$ scenarios with $K=10$ arms with means $\mu_i = i/10$. The covariance matrix $C$ is defined as follows: for $i\in \intr{K}$, $C_{ii} = 1$ and for $i\neq j$: $C_{ij} = \rho$. We consider $4$ scenarios with  the correlation: $\rho \in \{0, 0.5, 0.7, 0.9\}$. Right figure: the average sample complexities with $K=16$ arms with means $\mu_i = i/10$. Arms in the same cluster have a correlation of $0.99$, and arms from different clusters are independent. We consider $4$ scenarios with different number of clusters: $n_{\text{cl}} \in \{8,4,2,1\}$. All clusters are of the same size.}
		\label{fig:1}
	\end{figure}
	%	\begin{figure}
	%		\includegraphics[width=7cm]{Figure_9}
	%		\caption{Left figure: the average sample complexities, in $4$ scenarios with $K=10$ arms with means $\mu_i = i/10$. The covariance matrix $C$ is defined as follows: for $i\in \intr{K}$, $C_{ii} = 1$ and for $i\neq j$: $C_{ij} = \rho$. We consider $4$ scenarios with difference values of the correlation: $\rho \in \{0, 0.5, 0.7, 0.9\}$. Right figure: the average sample complexities, in $4$ scenarios with $K=16$ arms with means $\mu_i = i/10$. Arms in the same cluster have correlation of $0.99$, arms from different clusters are independent. We consider $4$ scenarios with different number of clusters: $n_{\text{clusters}} \in \{8,4,2,1\}$. All clusters are of the same size.}
	%		\label{fig:1}
	%	\end{figure}
	The second experiment aims to demonstrate that Algorithm~\ref{algo:1} adapts to the covariance between sub-optimal arms, as indicated by bound~\eqref{eq:g1}
	%\eqref{eq:b2_1} 
	in Theorem~\ref{thm:2},% \ref{thm:0}, 
	rather than solely adapting to the correlation with the optimal arm as shown by~\eqref{eq:g2}. %\eqref{eq:b2_2}. 
	In this experiment, we consider that the arms are organized into clusters, where each pair of arms within the same cluster exhibits a high correlation (close to 1), while arms from different clusters are independent. According to bound~\eqref{eq:g1} (see also discussion after Theorem~\ref{thm:2}) %\eqref{eq:b2_1}, 
	we expect to observe a gain of up to a factor corresponding to the number of arms per cluster compared to algorithms that do not consider covariance. (Since the total number of arms is kept fixed, the number of arms per cluster scales
	as the inverse of the number of clusters.)
	
	Figure~\ref{fig:1} illustrates the results, showing that the performance of Pairwise-BAI improves as the number of clusters decreases (indicating a larger number of correlated arms). This suggests that increasing the number of correlated sub-optimal arms, which are independent from the optimal arm, still leads to significant performance improvement. These findings support the idea that Pairwise-BAI and Pairwise-BAI+ exhibit behavior that aligns more closely with bound \eqref{eq:g1} %\eqref{eq:b2_1} 
	in Theorem~\ref{thm:2}, rather than bound~\eqref{eq:g2}. %\eqref{eq:b2_2}.
	
	In both experiments, we observe that Pairwise-BAI+ performs worse compared to Pairwise-BAI, indicating that, empirically, in the given scenarios, continuing to sample sub-optimal arms does not contribute to improved performance. While this modification provides better theoretical guarantees, it may not lead to empirical performance improvements in general scenarios.
	
	\vspace{-0.3cm}
	\section{Conclusion and future directions}
	\vspace{-0.2cm}
	%Our objective throughout this paper has been to address the challenge of best arm identification with fixed confidence in a relaxed setting that permits simultaneous queries per round. We have demonstrated that by considering the inherent correlation between arms, it is possible to achieve our objective with reduced query complexity.
	
	%In certain scenarios, it can be advantageous to compare each candidate arm with a convex combination of the remaining arms. This implies for an arm $i \in \intr{K}$, the existence of a weight vector $w = (w_{i})_{i \in \intr{K}}$, with non-negative entries such that $\norm{w}_1\le 1$,  such that $\Var(X_i - \langle w, \bm{X}\rangle) \ll \min_{j \in \intr{K}} \Var(X_i - X_j)$. This suggests that significant improvement can be attained by moving beyond pairwise comparisons. We have developed an algorithm based on this concept and provided the corresponding theoretical guarantees in Section~\ref{sec:convex} of the appendix.
	
	This work gives rise to several open questions. Firstly, the presented lower bounds take into account partially the covariance between arms. It would be interesting to explore the development of a more precise lower bound that can adapt to any covariance matrix of the arms. Additionally, in terms of the upper bound guarantees, our focus has been on pairwise comparisons, along with an algorithm that compares candidate arms with convex combinations of the remaining arms (Section~\ref{sec:convex}). An interesting direction for further research would involve extending this analysis to an intermediate setting, involving comparisons with sparse combinations.

    \paragraph{Acknowledgements:}  This work 
    is supported by ANR-21-CE23-0035 (ASCAI) and
    ANR-19-CHIA-0021-01 (BISCOTTE). This work was conducted while EM Saad was in INRAE Montpellier.

	%%%%%%%%%%%%%%%%%%%%%%%%%%%%%%%%%%%%%%%%%%%%%%%%%%%%%%%%%%%%
	\newpage
	\bibliographystyle{plain}
	\bibliography{arxiv_vof}

	\newpage
	
	\appendix
	\begin{center}
		{\bf \Large Supplementary material for:\\ Covariance adaptive best arm identification.}
	\end{center}
	
	\vspace{1cm}

	\section{Notation}\label{sec:not}
	
	\begin{itemize}
		\item Let $\bm{X} = (X_1, \dots, X_K)$ denote the vector of variables associated to the arms.
		\item Let $\bm{\mu} = (\mu_i)_{i \in \intr{K}}$ denote the vector of arms' means.
		\item For each $i, j \in \intr{K}$, let $V_{ij} = \Var(X_i-X_j)$.
		\item For each round $t \ge 1$ let $\bm{X}_t = (X_{1,t}, \dots, X_{K,t})$ denote the rewards sampled by the environment at round $t$.
		\item Let $\hat{\mu}_{i,t}$ denote  empirical mean of samples pulled from arm $i$ up to round $t$:
		\[
		\hat{\mu}_{i,t} := \frac{1}{t} \sum_{s=1}^{t} X_{i,s}.	
		\]
		Denote $\hat{\bm{\mu}}_t = (\hat{\mu}_{i,t}, \dots, \hat{\mu}_{K,t})$.  
		\item Let $(Y_t)_t$ denote a sequence of random variables distributed following $Y$:
		\[
		\widehat{V}_{t}(Y) =  \frac{1}{t(t-1)} \sum_{1\le u<v \le t}  (Y_{u} - Y_{v})^2. 
		\]
		\item For $i,j \in \intr{K}$, define the empirical variance for $(X_i-X_j)$ as follows:
		\begin{equation}\label{eq:def_sample_v}
		\widehat{V}_{ij,t} := \frac{1}{t(t-1)} \sum_{1\le u<v \le t} \left( (X_{i,u}-X_{j,u}) - (X_{i,v}-X_{j,v}) \right)^2.
		\end{equation}
		%\item For $i, j \in \intr{K}$, let $\hat{d}_{ij,t} := \hat{d}_{t}(X_i, X_j)$, $d_{ij} = d(X_i, X_j)$.
		\item Define $\delta_t := \delta/(2K^2t(t+1))$ and  $\alpha(t,\delta) := \sqrt{\frac{\log(\delta_t^{-1})}{t-1}}$.
		%\item Let $\beta(t,\delta)=\exp(2\alpha(t,\delta)+1)$, if $\alpha(t,\delta)\ge 1/3$, and $\beta(t,\delta) = 1/(1-2\alpha(t,\delta))$, otherwise.
		\item For bounded variables, define: 
		\[
		\hat{\Delta}_{ij}(t,\delta) := \hat{\mu}_{j,t} - \hat{\mu}_{i,t}  - \frac{3}{2}\alpha(t,\delta)\sqrt{2\widehat{V}_{ij,t}}  - 9~\alpha^2(t,\delta).
		\]
		
		\item For Gaussian variables, define:
		\[
		\hat{\Delta}'_{ij}(t,\delta) := \hat{\mu}_{j,t} - \hat{\mu}_{i,t}  - \frac{3}{2}\alpha(t,\delta)\sqrt{2f(\alpha(t,\delta))\widehat{V}_{ij,t}}.
		\]

		\item For bounded variables, define: 
		\[
		\hat{\Gamma}_{i}(w,t,\delta) := \langle w,\hat{\bm{\mu}}_{t} \rangle - \hat{\mu}_{i,t}  - 2\sqrt{2K\widehat{V}_{t}(X_i, \langle  w,  \bm{X} \rangle )}\alpha(t,\delta) - 14 K\norm{w-e_i}_1~\alpha^{2}(t,\delta),
		\]
		where $w \in \bm{B}_1^+$ and $\bm{B}_1^+ := \{ w, w \in [0,1]^K \text{ and } \norm{w}_1 = 1\}$.
		\item Define for $S \subseteq \intr{K}$ and $t \ge 1$
		\begin{equation}\label{def:B_1}
		\bm{B}_1^+(S) := \{ w \in \bm{B}_1^+  \text{ such that: } \text{supp}(w) \subseteq S\}.
		\end{equation}

		\item For bounded arms problem, define for $i,j \in \intr{K}$:
		\[
		\Lambda_{ij} := \left\{
		\begin{array}{ll}
		+\infty & \mbox{if } \mu_j \le \mu_i \\
		\frac{V_{ij}}{(\mu_j - \mu_i )^2} + \frac{3}{\mu_j- \mu_i }  & \mbox{otherwise }
		\end{array}
		\right.
		\]
		
		\item For Gaussian arms problem, define for $i,j \in \intr{K}$:
		\[
		\Lambda'_{ij} := \left\{
		\begin{array}{ll}
		+\infty & \mbox{if } \mu_j \le \mu_i \\
		\frac{V_{ij}}{(\mu_j - \mu_i )^2}  & \mbox{otherwise }
		\end{array}
		\right.
		\]
		\item Observe the quantities $\Lambda_{ij}$ and $\Lambda'_{ij}$ defined in the previous sections are slightly different from the above quantities.
		\item For bounded arms problem, define for $i \in \intr{K}$ and $w\in \mathbb{B}_1^+$:
		\[
		\Xi_{i}(w) := \left\{
		\begin{array}{ll}
		+\infty & \mbox{if } \langle w, \bm{\mu} \rangle \le \mu_i \\
		\max\left\lbrace \frac{\Var( \langle \bm{X}, w \rangle- X_i )}{(\langle w, \bm{\mu} \rangle - \mu_i )^2}, \frac{3\norm{w-e_i}_1}{\langle w, \bm{\mu} \rangle- \mu_i }\right\rbrace  & \mbox{otherwise }
		\end{array}
		\right.
		\]
		
		%\item Let $i^*$ denote the optimal arm. For $i \in \intr{K} \setminus \{i^*\}$, define
		%\[
		%\Lambda_{i}^* := \min_{j \in \intr{K}} \Lambda_{ij} \quad \text{ and } \quad \left(\Lambda'_{i} \right)^* := \min_{j \in \intr{K}} \Lambda'_{ij} \quad \text{ and } \quad \Xi^*_{i} := \min_{w \in \bG} \Xi_i(w).
		%\]
		%\item For $i \in S^*$, let $\Lambda_i^{(1)} = \max_{j \notin S^*} \Lambda_j^{(1)}$ and $\Lambda_i^{(2)} = \max_{j \notin S^*} \Lambda_j^{(2)}$.
		\item Notation for Algorithms~\ref{algo:0} and~\ref{algo:1}: In round $t$, let $S_t$ denote the set of candidate arms and $C_t$ the set of queried arms at round $t$.
	\end{itemize}
	
	\section{Additional Results for bounded variables: comparison to convex combination} \label{sec:convex}
	In this section we consider that arms distributions are bounded by $1$. We adopt the same notation introduced in the main body and add the following: Let $\bm{B}_1^+ \subset \mathbb{R}^K$ denote the set of vectors $w$ with non-negative entries such that $\norm{w}_1 = 1$. Let $\bm{X}:= (X_1, \dots, X_K)$ and $\langle ., .\rangle$ denote the euclidean scalar product in $\mathbb{R}^K$. For a subset $A\subset \intr{K}$, we denote by $\bm{B}_1^+(A)$ the set of elements in $\bm{B}_1^+$ with support in $A$.
	
	While in the previous sections the main idea of the presented procedures is to perform pairwise comparisons between arms, we consider here that for some classes of covariance matrices between the arms, it may is beneficial to perform sequential tests comparing the candidate arms with convex combinations of the non-eliminated arms. For example, for a sub-optimal arm $i$, it is possible to have for some weights vector $\bm{w}$, with support in $\intr{K}\setminus \{i\}$: $\Var(X_i - \langle w, \bm{X}\rangle) \ll \min_{j\in \intr{K}} \Var(X_i-X_j)$. In this case, it is advantageous to eliminate arm $i$ through a comparison with the combination $\langle w, \bm{X}\rangle$ instead of pairwise comparisons, as concluding that $\mathbb{E}[X_i - \langle w, \bm{X} \rangle] <0$ for some $w \in \bm{B}_1^+$ signifies that arm $i$ is sub-optimal.
	
	The approach used in this section shares similarities with the preceding methodology. More precisely, we develop an empirical second-order concentration inequality over the differences $(X_i - \langle w, \bm{X}\rangle)$ for $i \in \intr{K}$ and $\bm{w} \in \bm{B}_1^+$, based on empirical Bernstein inequality and a covering argument over $\bm{B}_1^+$. We define the following quantity: for $i \in \intr{K}$ and $w \in \bm{B}_1^+$.
	\[
	\hat{\Gamma}_{i}(w,t,\delta) := \langle w,\hat{\bm{\mu}}_{t} \rangle - \hat{\mu}_{i,t}  - 2\sqrt{2K\widehat{V}_{t}(X_i- \langle  w,  \bm{X} \rangle )}\alpha(t,\delta) - 14 K\norm{w-e_i}_1~\alpha^{2}(t,\delta).
	\]
	
	Lemma~\ref{lem:conc2} shows that if $\hat{\Gamma}_{i}(w,t,\delta) >0$, then $\mathbb{E}\left[\langle w, \bm{X} \rangle\right] > \mu_i$.

	\begin{algorithm} 
		%\centering
		\caption{Convex-BAI \label{algo:0}}
		\begin{algorithmic}[1]
			\STATE \textbf{Input}  $\delta$.
			\STATE \textbf{Initialization:}
			\STATE \hspace{1cm} Query all arms for $2$ rounds and compute the empirical means vector $\hat{\bm{\mu}}_t$, $t \gets 3$. 
			\STATE \hspace{1cm} $S \gets \intr{K}$, \qquad /*\texttt{Set of candidate arms}*/
			\STATE \hspace{1cm} $C \gets \intr{K}$, \qquad /*\texttt{Set of queried arms}*/
			%\STATE \hspace{1cm} $\mathbf{\hat{\mu}}_0 \gets (0, \dots, 0)$, $t \gets 1$.
			\WHILE{ $\lvert S \rvert>1$ } 
			\STATE Jointly query all the experts in $ C$.
			\STATE Update $\hat{\bm{\mu}}_t$ and compute $\sup_{w \in \mathbb{B}_1^+(C )} \hat{\Gamma}_i(w, t, \delta)$ for each $i \in S$.
			\FOR{$i \in S$}
			\IF{$\sup_{w \in \bm{B}_1^+(C)} \hat{\Gamma}_i(w, t, \delta)>0$}
			\STATE 	Eliminate $i$ from $S$: $S \gets S \setminus \{i\}$. \qquad/* \texttt{$i$ is  sub-optimal} */
			\STATE Activate a counter to eliminate $i$ from $C$ at round: $98~t$. \label{line:algo2}
			\ENDIF
			\ENDFOR
			\STATE \bf Increment $t$.
			\ENDWHILE
			\STATE  {\bf Return} $S$. 
		\end{algorithmic}
	\end{algorithm}
	
	\begin{remark}
		In Algorithm~\ref{algo:0}, we did not specify a method to perform the test: $\sup_{w \in \mathbb{B}_1^+(C_t)} \hat{\Gamma}_i(w, t, \delta)>0$. Several developments can be envisioned, such that using methods for convex optimization over a simplex.
	\end{remark}

	Finally, Theorem~\ref{thm:1} below provides guarantees on the strategy of Algorithm~\ref{algo:0}. Where tests are performed for each expert against convex combination of all arms. 
	
	%	\begin{equation}\label{eq:gami}
	%		\forall i \in \intr{K} \setminus \{i^*\}, \text{ let } \Xi^*_i := \min_{w \in \bG} \Xi_i(w) \text{ and } \Xi^* := \max_{i \in \intr{K}} \Xi^*_i.
	%	\end{equation}
	
	\begin{theorem}\label{thm:1}
		Suppose Assumption~\ref{a:bounded} hold. Consider Algorithm~\ref{algo:0}, with input $\delta \in (0,1)$. With probability at least $1-\delta$: the procedure selects the best arm $i^*$, and the total number of queries $N$ satisfies:
		\[
		N \le c \log(K\Lambda\delta^{-1})~K \sum_{i \in \intr{K}\setminus \{i^*\}} \min_{\substack{\bm{w} \in \bm{B}_1^{+}:\\ \langle \bm{w}, \bm{\mu} \rangle > \mu_i }} \left\lbrace \frac{\Var\left(X_i - \langle \bm{w}, \bm{X}\rangle\right)}{( \langle \bm{w}, \bm{\mu}\rangle - \mu_i)^2}+\frac{\norm{\bm{w}-e_i}_1}{ \langle \bm{w}, \bm{\mu}\rangle - \mu_i}\right\rbrace.
		\]
	\end{theorem}

	\section{Concentration lemmas for bounded variables}
	
	Define the event $(\mathcal{B}_1)$ for pairwise comparisons:
	
	$\qquad\forall t \ge 2, \forall~i,j \in \intr{K}:$
	\begin{subnumcases}{\label{eq:eventa2} }
	\left| \left(\hat{\mu}_{i,t} - \hat{\mu}_{j,t}\right) - \left(\mu_i - \mu_j\right) \right| \le \alpha(t,\delta)\sqrt{2\widehat{V}_{ij,t}} + 6~\alpha^2(t,\delta)  \label{conc3}\\
	\lvert \sqrt{\widehat{V}_{ij,t}} - \sqrt{V_{ij}}\rvert \le  2\sqrt{2}\alpha(t, \delta) \label{conc4},
	\end{subnumcases}
	
	\noindent where $\widehat{V}_{ij,t}$ is the empirical variance of the sequence $(X_{i,t} - X_{j,t})_{t\ge 1}$, $V_{ij} = \Var(X_i - X_j)$, $\alpha(t,\delta) = \sqrt{\log(\delta_t^{-1})/(t-1)}$ and $\delta_t = 2K^2\delta/(t(t+1))$.
	
	\noindent Define the event $(\mathcal{B}_2)$ for comparisons with convex combinations:
	
	$\qquad\forall t \ge 2, \forall~i \in \intr{K}, \forall w \in \mathbb{B}_1^+(C_t)$:
	\begin{subnumcases}{\label{eq:eventa1}
	}
	\left| \left(\langle w, \hat{\bm{\mu}_t} \rangle -  \hat{\mu}_{i,t}\right) - \left(\langle w, \bm{\mu} \rangle - \mu_i\right) \right| \le  \sqrt{2K\widehat{V}_{t}(X_i -  \langle w, \bm{X} \rangle )}~\alpha(t,\delta)   + 7K\norm{w-e_i}_1 \alpha^2(t,\delta) \label{conc1}\\
	\lvert \sqrt{\widehat{V}_{t}(X_i -  \langle w, \bm{X} \rangle )} - \sqrt{\Var\left(X_i - \langle w, \bm{X} \rangle\right)} \rvert \le 3\sqrt{K}\norm{w-e_i}_1~ \alpha(t, \delta), \label{conc2}
	\end{subnumcases}
	
	\noindent where $\mathbb{B}_1^+(C_t )$ is defined in Section~\ref{sec:not} \eqref{def:B_1}.

	We show that events $(\mathcal{B}_1)$ and $(\mathcal{B}_2)$, defined in \eqref{conc1}, \eqref{conc2} and \eqref{conc3}, \eqref{conc4} respectively, hold with high probability.
	%\textcolor{red}{TODO:} Fix a convexity issue in the proof
	\begin{lemma}\label{lem:conc1}
		We have $\mathbb{P}(\mathcal{B}_1) \ge 1-3\delta$.
	\end{lemma}
	
	\begin{proof}
		The first inequality is a direct consequence of empirical Bernstein's inequality (Theorem~\ref{conc:bern}) applied to the sequence of i.i.d variables $(X_{i,s} - X_{j,s})_{s \le t}$, and using a union bound over $i,j \in \intr{K}$ and $t \ge 2$.
		The second inequality of event $(\mathcal{B}_1)$ is a direct consequence of Theorem~\ref{thm:concvar}.

	\end{proof}
	
	\begin{lemma}\label{lem:conc2}
		We have $\mathbb{P}(\mathcal{B}_2) \ge 1-4\delta$.
	\end{lemma}
	
	\begin{proof}
		Let $\bm{B}_1$ denote the unit ball with respect to $\norm{.}_1$ in $\mathbb{R}^K$. We will show that:
		$\qquad\forall t \ge 2, \forall~i \in \intr{K}, \forall \bm{u} \in \mathbb{B}_1$:
		\begin{subnumcases}{\label{eq:eventa3}
		}
		\left| \langle \bm{u}, \hat{\bm{\mu}_t}-\bm{\mu} \rangle \right| \le  \sqrt{2K\widehat{V}_{t}(  \langle \bm{u}, \bm{X} \rangle )}~\alpha(t,\delta)   + 7K \alpha^2(t,\delta) \label{conc1u}\\
		\lvert \sqrt{\widehat{V}_{t}(  \langle \bm{u}, \bm{X} \rangle )} - \sqrt{\Var\left( \langle \bm{u}, \bm{X} \rangle\right)} \rvert \le 3\sqrt{K}~ \alpha(t, \delta), \label{conc2u}.
		\end{subnumcases}
		The result follows by taking $\bm{u} = (w-e_i)/\norm{w-e_i}_1$ or $\bm{u}=e_i$.

		We use a standard covering argument. Recall that the $\epsilon-$covering number for $\bm{B_1}$, with respect to $\norm{.}_1$, is upper bounded by $(3/\epsilon)^K$ (Lemma 5.7 in \citealp{wainwright2019high}). 
		
		Fix $\delta \in (0,1)$. For each $i \in \intr{K}$, $t\ge 1$, let $\epsilon_t>0$ be a parameter to be specified later. 
		Let $\mathcal{N}_t$ be an $\epsilon_t$-cover of the set of $\mathbb{B}_1$, with respect to $\norm{.}_1$. We will first prove that the event defined in the beginning of the proof is true for all $u\in \mathcal{N}_t$, then using the triangle inequality, we will prove the inequality for any $u\in \mathbb{B}_1$.
		
		Let $i \in \intr{K}$ and $u \in \mathcal{N}_t$. Applying Theorem~\ref{conc:bern} to the sequence of i.i.d variables $(\langle u, \bm{X}_s \rangle)_{s \le t}$ bounded by $1$ 
		%and bounding the empirical variance similarly to \eqref{eq:varbound}
		, we have with probability at least $1 - \delta$, 
		\[
		\left| \langle u, \hat{\bm{\mu}_t}-\bm{\mu} \rangle  \right| \le   \sqrt{\frac{2\log(3\delta^{-1})\hat{V}_{t}(  \langle u, \bm{X} \rangle )}{t}}   + 6 \frac{\log(3\delta^{-1})}{t},
		\]
		where $\hat{V}_{t}( \langle u, \bm{X} \rangle )$ denotes the empirical variance of $( \langle u, \bm{X} \rangle)$ at round $t$.
		Using a union bound over $t \ge 1$, $i \in \intr{K}$ and $u\in \mathcal{N}_t$, we have with probability at least $1 - \delta$: $\forall t \ge 1, i \in \intr{K}, u \in \mathcal{N}_t$:
		\begin{align}
		\left| \langle u, \hat{\bm{\mu}_t}-\bm{\mu} \rangle  \right| &\le  \sqrt{2}~\sqrt{\frac{\log(\lvert \mathcal{N}_t \rvert\delta_t^{-1})}{t-1}\hat{V}_{t}(  \langle u, \bm{X} \rangle )}   + 6 \frac{\log(\lvert \mathcal{N}_t \rvert\delta_t^{-1})}{t-1} \notag \\
		&\le \alpha(t, \delta/\lvert \mathcal{N}_t \rvert)\sqrt{2\hat{V}_{t}(\langle u, \bm{X} \rangle )}   + 6 \alpha^2(t, \delta/\lvert \mathcal{N}_t \rvert), \label{eq:b2_1}
		\end{align}
		where $\delta_t = \delta/(2K^2t(t+1))$.
		
		Applying Theorem~\ref{thm:concvar} to the sequence $(\langle u, \bm{X} \rangle)$ at round $t$, we have with probability at least $1-2\delta$:
		\[
		\lvert \sqrt{\hat{V}_t( \langle u, \bm{X} \rangle)} - \sqrt{\Var( \langle u, \bm{X} \rangle)} \rvert \le 2\sqrt{\frac{2\log(\delta^{-1})}{t-1}}.
		\] 
		Now, we use a union bound over $t \ge 1$, $i \in \intr{K}$  and $u \in \mathcal{N}_t$ to obtain with probability at least $1-\delta$: $\forall t \ge 1, i \in \intr{K}, u \in \mathcal{N}_t$
		\begin{equation}\label{eq:b2_2}
		\lvert \sqrt{\hat{V}_t(\langle u, \bm{X} \rangle)} - \sqrt{\Var( \langle u, \bm{X} \rangle)} \rvert \le 2\sqrt{2}\alpha(t, \delta/\lvert \mathcal{N}_t \rvert).
		\end{equation}
		
		\noindent To wrap up, fix $t\ge 1$ and let $u' \in \mathbb{B}_1$. Since $\mathcal{N}_t$ is a covering for $\mathbb{B}_1$, we have: $\exists u \in \mathcal{N}_t$ such that $\norm{u' - u}_1 \le \epsilon_t$. 
		
		\noindent Hence 
		\begin{align*}
		\left| \langle u', \hat{\bm{\mu}}_t - \bm{\mu}\rangle \right| &\le \left| \langle u, \hat{\bm{\mu}}_t - \bm{\mu}\rangle \right| + \left| \langle  u-u', \hat{\bm{\mu}}_t - \bm{\mu} \rangle \right|\\
		&\le \sqrt{2\hat{V}_{t}( \langle u, \bm{X} \rangle )} \alpha(t, \delta/\lvert \mathcal{N}_t \rvert)   + 6 \alpha^2(t, \delta/\lvert \mathcal{N}_t \rvert) + \epsilon_t,
		\end{align*}
		where we used \eqref{eq:b2_1} and $\norm{u-u'}_1 \le \epsilon_t$. Moreover, we have
		%		\[
		%		\sqrt{\hat{V}_{t}(X_i ,  \langle w', \bm{X} \rangle )} \le \sqrt{\hat{V}_{t}(X_i ,  \langle w, \bm{X} \rangle )} + B\epsilon_t.
		%		\]
		\begin{align*}
		\sqrt{\hat{V}_{t}( \langle u, \bm{X} \rangle )} &\le \sqrt{\hat{V}_{t}( \langle u', \bm{X} \rangle )} + \sqrt{\hat{V}_{t}(  \langle u-u', \bm{X} \rangle )}\\
		&\le \sqrt{\hat{V}_{t}( \langle u', \bm{X} \rangle )} + \norm{u-u'}_1 \\
		&\le  \sqrt{\hat{V}_{t}(\langle u', \bm{X} \rangle )} + \epsilon_t.
		\end{align*}
		Therefore 
		\begin{align}
		\left| \langle u', \hat{\bm{\mu}}_t - \bm{\mu}\rangle \right| &\le \sqrt{2\hat{V}_{t}(\langle u', \bm{X} \rangle )}~\alpha(t, \delta/\lvert \mathcal{N}_t \rvert)   \notag\\
		&\qquad \quad  + 6 \alpha^2(t, \delta/\lvert \mathcal{N}_t \rvert) + \epsilon_t \left(1+\sqrt{2} \alpha(t, \delta/\lvert \mathcal{N}_t \rvert)\right).\label{eq:inter0} 
		\end{align}

		\noindent We choose $\epsilon_t = 3\delta_t^{\frac{K-1}{K}} = 3\left( \frac{\delta}{2K^2t(t+1)}\right)^{\frac{K-1}{K}}$, therefore
		\begin{align}
		\alpha(t, \delta/\lvert \mathcal{N}_t \rvert) &= \sqrt{\frac{\log\left((3/\epsilon_t)^K\delta_t^{-1}\right)}{t-1}}\notag\\
		&=\sqrt{\frac{\log\left(\delta_t^{-(K-1)}\delta_t^{-1}\right)}{t-1}}\notag\\
		&= \sqrt{K}~\alpha(t,\delta)\label{eq:inter1}. 
		\end{align}

		\noindent Furthermore, we have
		\begin{align}
		\epsilon_t \left(1+\sqrt{2} \alpha(t, \delta/\lvert \mathcal{N}_t \rvert)\right) &\le 3(\delta_t)^{\frac{K-1}{K}} \left(1+\sqrt{K}~ \alpha(t, \delta)\right)\notag\\
		&\le 3(\delta_t)^{1/2} \left(1+\sqrt{K}~ \alpha(t, \delta)\right)\notag\\
		&\le 3\delta^{1/2}\left(\frac{1}{K\log(\delta_t^{-1})}+\frac{2\sqrt{K}}{K\sqrt{(t-1)\log(\delta_t^{-1})}} \right) \frac{\log(\delta_t^{-1})}{t-1}\notag\\
		%&\le \frac{3\delta}{2K^2t(t+1)}+\frac{3K}{\sqrt{t-1}} \notag\\
		&\le \alpha^2(t,\delta). \label{eq:inter2}
		%			&\le B\frac{K \log(2K^2\delta_t^{-1})}{t}~\frac{\delta}{t+1} \frac{1+2\sqrt{K\log(2K^2\delta_t^{-1})}}{K \log(2K^2\delta_t^{-1})}\\
		%			&\le B\frac{K \log(2K^2\delta_t^{-1})}{t}\\
		%			&\le BK\alpha^2(t, \delta).
		\end{align}
		%		Therefore,
		%		\begin{equation}\label{eq:inter2}
		%			B\epsilon_t \left(1+\sqrt{2K} \alpha(t, \epsilon_t\delta/3)\right) \le BK\alpha^2(t, \delta).
		%		\end{equation}
		\noindent We plug \eqref{eq:inter1} and \eqref{eq:inter2} into \eqref{eq:inter0}, and obtain that with probability at least $1-\delta$
		\begin{equation}\label{eq:conc_1}
		\left| \langle u', \hat{\bm{\mu}}_t - \bm{\mu}\rangle \right| \le  \sqrt{2K\widehat{V}_{t}(  \langle u', \bm{X} \rangle )} \alpha(t, \delta)   + 7 K \alpha^2(t, \delta). 
		\end{equation}
		We proceed similarly for the second concentration inequality. We have with probability at least $1-\delta$
		\begin{align}
		\lvert \sqrt{\widehat{V}_t( \langle u', \bm{X} \rangle)} - \sqrt{\Var( \langle u', \bm{X} \rangle)} \rvert &\le  \lvert \sqrt{\widehat{V}_t( \langle u, \bm{X} \rangle)} - \sqrt{\Var( \langle u, \bm{X} \rangle)} \rvert + \epsilon_t\notag\\
		&\le  2\sqrt{2}\alpha(t, \delta/\lvert \mathcal{N}_t \rvert) + \epsilon_t \notag\\
		&\le 3\sqrt{K} \alpha(t, \delta) \label{eq:conc_2}.
		\end{align}
		
		\noindent	We conclude by combining \eqref{eq:conc_1} and \eqref{eq:conc_2}.
	\end{proof}

	\section{Concentration lemmas for Gaussian variables}
	
	Recall the definition $\alpha(t,\delta) = \sqrt{\log(2K^2t(t+1)\delta^{-1})/(t-1)}$. Define the function $f$ for positive numbers as follows:
	
	\[
	f(x) := \left\{
	\begin{array}{ll}
	\exp(2x+1) & \mbox{if } x \ge 1/3 \\
	\frac{1}{1-2x} & \mbox{otherwise, }
	\end{array}
	\right.
	\]

	%[TODO: compute closed formula for $\beta$ ?]

	Define the event $(\mathcal{A}_1)$:$\qquad\forall t \ge 2$, $\forall~i,j \in \intr{K}:$
	\begin{subnumcases}{\label{eq:eventa2-G} }
	\left| \left(\hat{\mu}_{i,t} - \hat{\mu}_{j,t}\right) - \left(\mu_i - \mu_j\right) \right| \le \sqrt{2f(\alpha(t,\delta))\widehat{V}_{ij,t}}~\alpha(t,\delta)  \label{conc3-G}\\
	\widehat{V}_{ij,t}  \le  V_{ij} \left(1+2\alpha(t,\delta)+2\alpha^2(t,\delta)\right) \label{conc4-G}\\
	V_{ij} \le \widehat{V}_{ij,t}~f(\alpha(t,\delta)). \label{conc5-G}
	\end{subnumcases}
	
	Lemma below shows that event $(\mathcal{A}_1)$ defined above holds with high probability.
	\begin{lemma}\label{lem:conc_g}
		We have $\mathbb{P}(\mathcal{A}_1) \ge 1-4\delta$.
	\end{lemma}
	\begin{proof}
		We start by proving inequalities \eqref{conc4-G} and \eqref{conc5-G}. We use Lemma~\ref{lem:conc_var_g} with $Y_t = (X_{i,t}-X_{j,t} - (\mu_i - \mu_j))/V_{ij}$. A union bound over $t\ge 2$ and $i,j \in \intr{K}$ gives with probability at least $1-\delta$
		\[
		\widehat{V}_{ij,t}  \le  V_{ij} \left(1+2\alpha(t,\delta)+2\alpha^2(t,\delta)\right).
		\]
		For inequality \eqref{conc5-G}, we apply the first result of Lemma~\ref{lem:conc_var_g} to the $(Y_t)$ defined above. Using a union bound, we have with probability at least $1-\delta$ 
		\[
		\widehat{V}_{ij,t} \ge V_{ij} \max\left\lbrace 1-2 \alpha(t,\delta); e^{-1} \exp\left(-2\alpha^2(t,\delta)\right) \right\rbrace.
		\]
		Inverting the inequality above leads to \eqref{conc5-G}.
		
		To prove \eqref{conc3-G}, we use Chernoff's concentration bound for Gaussian variables (and a union bound), with probability at least $1-\delta$: for all $t\ge 2$ and $i,j \in \intr{K}$: 
		\begin{equation}\label{eq:gg1}
		\lvert (\hat{\mu}_{i,t} - \hat{\mu}_{j,t}) - \left(\mu_i - \mu_j\right) \rvert \le \sqrt{2V_{ij}} \alpha(t,\delta).
		\end{equation}
		We plug-in the bound \eqref{conc5-G} to obtain the result.	
		
	\end{proof}

	\section{Key lemmas}

	\begin{lemma}\label{lem:elim_2}
		If $(\mathcal{B}_1)$ defined in \eqref{eq:eventa2} holds, we have the following: 
		
		For any $i \in \intr{K}$, if there exists $t \ge 1$ and $j \in \intr{K}$ such that $\hat{\Delta}_{ij}(t,\delta) > 0$, then $i \neq i^*$.
		
		Moreover, if $(\mathcal{A}_1)$ defined in \eqref{eq:eventa2} holds, we have the following: 
		
		For any $i \in \intr{K}$, if there exists $t \ge 2$ and $j \in \intr{K}$ such that $\hat{\Delta}'_{ij}(t,\delta) > 0$, then $i \neq i^*$.

	\end{lemma}
	
	\begin{proof}
		Suppose that $(\mathcal{B}_1)$ is true. Let $t \ge 2$, $i,j \in \intr{K}$. We have
		\begin{align*}
		\mu_j - \mu_i &= \hat{\Delta}_{ij}(t,\delta) + \mu_j - \mu_i - (\hat{\mu}_{j,t} - \hat{\mu}_{i,t} ) + \frac{3}{2}\sqrt{2\widehat{V}_{ij,t}} \alpha(t,\delta) + 9 \alpha^2(t,\delta)\\
		&\ge \hat{\Delta}_{ij}(t,\delta),
		\end{align*}
		where we used \eqref{conc3-G}. Finally, if $\hat{\Delta}_{ij}(t,\delta) > 0$, we have $\mu_j > \mu_i $.
		
		Following the exact same steps we have the result for Gaussian variables (the second claim).
	\end{proof}
	
	\begin{lemma}\label{lem:elim_1}
		If $(\mathcal{B}_2)$ defined in \eqref{eq:eventa1} holds, we have the following: 
		
		For any $i \in \intr{K}$, if there exists $t \ge 1$ and $w \in \mathbb{B}_1^+(C_t)$ such that: $\hat{\Gamma}_{i}(w,t,\delta) > 0$, then $i \neq i^*$.
	\end{lemma}
	\begin{proof}
		Suppose that $(\mathcal{B}_2)$ is true. Let $t \ge 1$, $i \in \intr{K}$ and $w \in \mathbb{B}_1^+(C_t)$. We have
		\begin{align*}
		\langle w, \bm{\mu} \rangle - \mu_i &= \hat{\Gamma}_{i}(w,t,\delta) + \langle w, \bm{\mu} \rangle - \mu_i - (\langle w, \hat{\bm{\mu}}_{t} \rangle - \hat{\mu}_{i,t} ) + 2\sqrt{2K\widehat{V}_{t}(X_i- \langle w, \bm{X} \rangle)}~ \alpha(t,\delta)\\
  &\qquad \qquad+ 14K\norm{w-e_i}_1 \alpha^2(t,\delta)\\
		&\ge \hat{\Gamma}_{i}(w,t,\delta),
		\end{align*}
		where we used \eqref{conc1}. If $\hat{\Gamma}_{i}(w,t,\delta) > 0$, we have $\langle w, \bm{\mu} \rangle > \mu_i $. Since $w$ is a vector of convex weights, we conclude that $\max_{j \in \text{supp}(w)} \mu_j \ge 	\langle w, \bm{\mu} \rangle > \mu_i$.  
	\end{proof}

	\begin{lemma}\label{lem:ord_2}
		If $(\mathcal{B}_1)$ defined in \eqref{eq:eventa1} holds, then for any $t \ge 1$, $i,j \in C_t$: such that $\mu_j>\mu_i$
		
		If $\hat{\Delta}_{ij}(t,\delta) > 0$, then
		\[
		t-1\ge \frac{1}{4}\log(\delta_t^{-1}) \Lambda_{ij}.
		\]
		Furthermore, if $\hat{\Delta}_{ij}(t,\delta) \le 0$, then
		\[
		t-1 \le \frac{25}{2}\log(\delta_t^{-1}) \Lambda_{ij}.
		\]
	\end{lemma}
	\begin{proof} 
		Suppose that $(\mathcal{B}_1)$ is true. Let $t\ge 1$, $i,j \in \intr{K}$. 
		Suppose that $\hat{\Delta}_{ij}(t,\delta) > 0$. We have
		\begin{align}
		\mu_j - \mu_i &= \hat{\Delta}_{ij}(t,\delta) - (\hat{\mu}_{j,t} - \hat{\mu}_{i,t} ) + \mu_j - \mu_i + \frac{3}{2}\sqrt{2\widehat{V}_{ij,t}}~ \alpha(t,\delta)  + 9 \alpha^2(t,\delta)\notag\\
		&\ge \hat{\Delta}_{ij}(t,\delta) - \sqrt{2V_{ij}}\alpha(t,\delta) - \frac{4}{3}\alpha^2(t,\delta) + \frac{3}{2}\sqrt{2\widehat{V}_{ij,t}}~ \alpha(t,\delta)  + 9 \alpha^2(t,\delta)\notag\\
		&\ge \hat{\Delta}_{ij}(t,\delta) + \sqrt{\frac{V_{ij}}{2}}~\alpha(t,\delta)+ \frac{3}{2} \sqrt{2\widehat{V}_{ij,t}} \alpha(t,\delta)-\frac{3}{2}\sqrt{2V_{ij}} \alpha(t,\delta)  + \frac{23}{3} \alpha^2(t,\delta) \notag\\
		&> \sqrt{\frac{V_{ij}}{2}}~\alpha(t,\delta)+\frac{5}{3} \alpha^2(t,\delta) ,\label{eq:invert1}
		\end{align}
		where we used Bennett's inequality (Theorem 3 in \cite{maurer2009empirical}) in the second line, \eqref{conc4}  with $\hat{\Delta}_{ij}(t,\delta)>0$ in the third line. 
		
		\noindent Solving inequality\eqref{eq:invert1}, gives
		
		\begin{align*}
		\alpha(t, \delta) &\le \frac{3}{14}\left(\sqrt{\frac{V_{ij}}{2}+ \frac{20}{3}(\mu_j - \mu_i ) }  - \sqrt{\frac{V_{ij}}{2}} \right)\\
		&= \frac{2(\mu_j - \mu_i)}{\sqrt{\frac{V_{ij}}{2}+ \frac{20}{3}(\mu_j - \mu_i ) }  + \sqrt{\frac{V_{ij}}{2}}}.
		\end{align*}
		Therefore, we have
		\begin{align*}
		t-1 &\ge \log(\delta_t^{-1}) \left( \frac{V_{ij}}{4(\mu_j - \mu_i )^2} + \frac{5/3}{\mu_j - \mu_i } \right)\\
		&\ge \frac{1}{4}\log(\delta_t^{-1})~\Lambda_{ij}.
		\end{align*}
		Which gives the first result.
		
		For the second bound, we proceed similarly. Suppose that $\hat{\Delta}_{ij}(t,\delta) \le 0$, we have:
		\begin{align}
		\mu_j - \mu_i &= \hat{\Delta}_{ij}(t,\delta_t) - (\hat{\mu}_{j,t} - \hat{\mu}_{i,t} ) + \mu_j - \mu_i + \frac{3}{2}\sqrt{2\widehat{V}_{ij,t}}~ \alpha(t,\delta)  + 9 \alpha^2(t,\delta)\notag\\
		&\le \hat{\Delta}_{ij}(t,\delta_t) + \sqrt{2V_{ij}}~\alpha(t,\delta) + \frac{4}{3}\alpha^2(t,\delta)+ \frac{3}{2}\sqrt{2\widehat{V}_{ij,t}}~ \alpha(t,\delta)  + 9 \alpha^2(t,\delta)\\
		&\le \hat{\Delta}_{ij}(t,\delta_t) + \frac{5}{2}\sqrt{2V_{ij}}~ \alpha(t,\delta)  + \frac{49}{3} \alpha^2(t,\delta) \notag\\
		&\le  \frac{5}{2}\sqrt{2V_{ij}}\alpha(t,\delta) +\frac{49}{3} \alpha^2(t,\delta) ,\label{eq:invert2}
		\end{align}
		
		Similarly to the previous case, we have:
		\begin{align*}
		\alpha(t, \delta) &\ge \frac{\sqrt{\frac{25}{2}V_{ij}+ \frac{196}{3}(\mu_j - \mu_i ) }  - \sqrt{\frac{25}{2}V_{ij}} }{98/3}\\
		&= \frac{2(\mu_j - \mu_i)}{\sqrt{\frac{25}{2}V_{ij}+ \frac{196}{3}(\mu_j - \mu_i ) }  + \sqrt{\frac{25}{2}V_{ij}}},
		\end{align*}
		which gives
		\begin{equation*}
		t-1 \le \log(\delta_t^{-1}) \left(\frac{25}{2} \frac{V_{ij}}{(\mu_i - \mu_j)^2} + \frac{98/3}{\mu_j - \mu_i}\right)
		\end{equation*}
		\noindent We conclude by inverting \eqref{eq:invert2} leading to: 
		\[
		t-1 \le \frac{25}{2}\log(\delta_t^{-1})~\Lambda_{ij}.
		\]
	\end{proof}
	
	\begin{lemma}\label{lem:ord_0}
		If $(\mathcal{A}_1)$ defined in \eqref{eq:eventa1} holds, then for any $t \ge 2$, $i,j \in C_t$, such that $\mu_j>\mu_i$
		
		If $\hat{\Delta}'_{ij}(t,\delta) > 0$, then
		\[
		t-1 \ge \frac{1}{2}\log(1/\delta_t)~\Lambda'_{ij} .
		\]
		Furthermore, if $\hat{\Delta}'_{ij}(t,\delta) \le 0$, then
		\[
		t-1 \le \frac{25 \log(1/\delta_t)}{\log\left(1+1/\Lambda'_{ij}\right)}.
		\]
		Moreover, if  $\hat{\Delta}'_{ij}(t,\delta) \le 0$, then: if $\Lambda'_{ij} <\frac{1}{5}$
		\[
		t-1 \le \frac{3\log(1/\delta_t)}{2\log\left(\frac{\sqrt{2}}{5\Lambda'_{ij}} \right)},
		\]
		while if $\Lambda'_{ij} \ge \frac{1}{5}$
		\[
		t-1 \le 25 \log(1/\delta_t) \Lambda'_{ij}.
		\]
	\end{lemma}
	\begin{proof} 
		Suppose that $(\mathcal{A}_1)$ is true. Let $t\ge 2$, $i,j \in \intr{K}$. 
		Suppose that $\hat{\Delta}'_{ij}(t,\delta_t) > 0$. We have
		\begin{align}
		\mu_j - \mu_i &= \hat{\Delta}'_{ij}(t,\delta) - (\hat{\mu}_{j,t} - \hat{\mu}_{i,t} ) + \mu_j - \mu_i + \frac{3}{2}\sqrt{2f(\alpha(t,\delta))\widehat{V}_{ij,t}}~ \alpha(t,\delta)\notag\\
		&\ge \hat{\Delta}'_{ij}(t,\delta) + \frac{1}{2}\sqrt{2V_{ij}}~\alpha(t,\delta) \notag\\
		&> \frac{1}{2}\sqrt{2V_{ij}}~\alpha(t,\delta) ,\label{eq:invert10}
		\end{align}
		where we used \eqref{conc3} in the second line and \eqref{conc4} with $\hat{\Delta}'_{ij}(t,\delta_t)>0$ in the third line. 
		
		\noindent Therefore inequality\eqref{eq:invert10}, gives
		
		\begin{equation*}
		\alpha(t, \delta) < \frac{\mu_j - \mu_i}{\sqrt{V_{ij}/2} }.
		\end{equation*}
		Therefore, we have
		\begin{align*}
		t-1 &\ge \frac{1}{2}\log(1/\delta_t) \frac{V_{ij}}{(\mu_i - \mu_j)^2} \\
		&= \frac{1}{2}\log(1/\delta_t)~\Lambda'_{ij} 
		\end{align*}
		Which gives the first result.
		
		For the second bound, suppose that $\hat{\Delta}'_{ij}(t,\delta) \le 0$, we have:
		\begin{align}
		\mu_j - \mu_i &= \hat{\Delta}'_{ij}(t,\delta_t) - (\hat{\mu}_{j,t} - \hat{\mu}_{i,t} ) + \mu_j - \mu_i + \frac{3}{2}\sqrt{2f(\alpha(t,\delta))\widehat{V}_{ij,t}}~\alpha(t,\delta) \notag\\
		&\le \hat{\Delta}'_{ij}(t,\delta_t)  + \frac{5}{2}\sqrt{2f(\alpha(t,\delta))\widehat{V}_{ij,t}}~\alpha(t,\delta) \notag\\
		&\le  \frac{5}{2}\sqrt{V_{ij}}\sqrt{2f(\alpha(t,\delta))(1+2\alpha(t,\delta)+2\alpha^2(t,\delta))}~\alpha(t,\delta) ,\label{eq:invert20}
		\end{align}
		Hence:
		\begin{equation}\label{eq:ineq_01}
		f(\alpha(t,\delta)) (\alpha^2(t,\delta)+2\alpha^3(t,\delta)+2\alpha^4(t,\delta)) \ge \frac{2(\mu_j - \mu_i)^2}{25V_{ij}}
		\end{equation}
		We consider two distinct cases:
		\paragraph{Case 1: $\lvert \mu_i - \mu_j \rvert> 5\sqrt{V_{ij}}$.} 
		Observe that in this case, inequality \eqref{eq:ineq_01} implies in particular that:
		\[
		f(\alpha(t,\delta)) (\alpha^2(t,\delta)+2\alpha^3(t,\delta)+2\alpha^4(t,\delta)) \ge 2.
		\]
		Recall that for a positive number $x<1/3$, we have $(x^2 + 2x^3+2x^4)/(1-2x^2)<2$. Hence, for the latter inequality to hold, we necessarily have  that: $\alpha(t,\delta) \ge 1/3$. Therefore, using the definition of $f$ we have $f(\alpha(t,\delta)) = \exp\left(2\alpha^2(t,\delta)+1\right)$, taking the logarithms in inequality \eqref{eq:ineq_01}:
		\begin{equation*}
		2\alpha^2(t,\delta) + \log\left(\alpha^2(t,\delta)+2\alpha^3(t,\delta)+2\alpha^4(t,\delta) \right) \ge 2 \log\left(\frac{\sqrt{2}\lvert \mu_i - \mu_j \rvert}{5\sqrt{V_{ij}}}\right),
		\end{equation*}
		Observe that for any $x>0$, we have $3x^2 > 2x^2 + \log(x^2+2x^3+2x^4)$. Therefore
		\begin{equation*}
		\alpha^2(t,\delta) > \frac{2}{3} \log\left(\frac{\sqrt{2}\lvert \mu_i - \mu_j \rvert}{5\sqrt{V_{ij}}}\right).
		\end{equation*}
		
		We conclude that 
		\begin{equation}\label{eq:ineq_02}
		t-1 < \frac{3/2\log(\delta_t^{-1})}{ \log\left(\frac{\sqrt{2}\lvert \mu_i - \mu_j \rvert}{5\sqrt{V_{ij}}}\right) }
		\end{equation}
		
		\paragraph{Case 2:$\quad \lvert \mu_i-\mu_j \rvert\le 5 \sqrt{V_{ij}}$.}
		If $\alpha(t,\delta) > 1/3$, then 
		\[
		t-1 < 9\log(\delta_t^{-1}).
		\]
		Otherwise, if $\alpha(t,\delta) \le 1/3$, we have using the definition of $f(\alpha(t,\delta))$ and inequality \eqref{eq:ineq_01} 
		\begin{align*}
		\frac{2(\mu_j - \mu_i)^2}{25V_{ij}} &\le \alpha^2(t,\delta)\frac{1+2\alpha(t,\delta)+2\alpha^2(t,\delta)}{1-2\alpha(t,\delta)}\\
		&< 2~\frac{\alpha^2(t,\delta)}{1-2\alpha(t,\delta)}.
		\end{align*}
		Inverting the inequality above in $t$, we obtain
		\begin{equation*}
		\alpha(t,\delta) \ge \sqrt{\frac{(\mu_j-\mu_i)^2}{25V_{ij}} \left(\frac{(\mu_j-\mu_i)^2}{25V_{ij}}+1 \right)} - \frac{(\mu_j-\mu_i)^2}{25V_{ij}}.
		\end{equation*}
		Therefore, we have
		\begin{equation*}
		\alpha^2(t,\delta) \ge \frac{(\mu_j-\mu_i)^2}{25V_{ij}}.
		\end{equation*}
		We conclude that if $\lvert \mu_i-\mu_j \rvert\le 5 \sqrt{V_{ij}}$:
		\begin{equation}\label{eq:ineq_04}
		t-1 < 25\log(1/\delta_t) \frac{V_{ij}}{(\mu_i-\mu_j)^2}.
		\end{equation}

		Now in order to unify the bounds obtained in Cases $1$ and $2$, observe that the function $f$ defined for positive numbers by
		\[
		f(x) := \left\{
		\begin{array}{ll}
		\frac{3}{\log(2/(25x^2))} & \mbox{if } x \le 1/5 \\
		25x^2 & \mbox{otherwise, }
		\end{array}
		\right.
		\]
		satisfies for any $0<x<1/5$
		\begin{equation*}
		f(x) \le \frac{10}{\log(1+\frac{1}{x^2})},
		\end{equation*}
		and for $x>1/5$
		\begin{equation*}
		f(x) \le \frac{25}{\log(1+\frac{1}{x^2})},
		\end{equation*}
		Therefore, we conclude that if $\hat{\Delta}'_{ij}(t,\delta) \le 0$, then
		\begin{equation*}
		t-1 \le \frac{25 \log(1/\delta_t)}{\log\left(1+1/\Lambda'_{ij}\right)}.
		\end{equation*}
	\end{proof}

	\begin{lemma}\label{lem:ord_1}
		If $(\mathcal{B}_2)$ defined in \eqref{eq:eventa1} holds, then for any $i \in S_t$, $t \ge 1$ and $w \in \mathbb{B}_1^+(C_t)$:
		
		If $\hat{\Gamma}_{i}(w,t,\delta) > 0$, then
		\[
		t \ge K\log(\delta_t^{-1})~\Xi_{i}(w).
		\]
		Furthermore, if $\hat{\Gamma}_{i}(w,t,\delta) \le 0$, then
		\[
		t \le 36K\log(\delta_t^{-1})~\Xi_{i}(w).
		\]
	\end{lemma}
	
	\begin{proof}
		Suppose that $(\mathcal{B}_2)$ is true. Let $t\ge 1$, $i \in S_t$ and $w \in \mathbb{B}_1^+(C_t)$. 
		Suppose that $\hat{\Gamma}_{i}(w,t,\delta) > 0$. We have
		\begin{align*}
		\langle w, \bm{\mu} \rangle - \mu_i &= \hat{\Gamma}_{i}(w,t,\delta) - (\langle w, \hat{\bm{\mu}}_{t} \rangle - \hat{\mu}_{i,t} ) + \langle w, \bm{\mu} \rangle - \mu_i + 2\sqrt{2K\widehat{V}_{t}(X_i- \langle w, \bm{X} \rangle)}~ \alpha(t,\delta)\\
        &\qquad \qquad + 14K\norm{w-e_i}_1 \alpha^2(t,\delta)\\
		&\ge \hat{\Gamma}_{i}(w,t,\delta) -\sqrt{2K\Var\left( X_i- \langle w, \bm{X} \rangle\right)} - \frac{4}{3}K\alpha^2(t,\delta) + 2\sqrt{2K\widehat{V}_{t}(X_i- \langle w, \bm{X} \rangle)}~ \alpha(t,\delta) \\
		&\qquad \qquad + 14K\norm{w-e_i}_1 \alpha^2(t,\delta)\\
		&\ge \hat{\Gamma}_{i}(w,t,\delta) +   \sqrt{2K\Var\left( X_i- \langle w, \bm{X} \rangle\right)}~\alpha(t,\delta)  + 7K\norm{w-e_i}_1 \alpha^2(t,\delta)\\
		&> \sqrt{2K \Var\left( X_i- \langle w, \bm{X} \rangle\right)}~\alpha(t,\delta)  +7K\norm{w-e_i}_1 \alpha^2(t,\delta) ,
		\end{align*}
		where we used Bennett's inequality in the second line and \eqref{conc2} with $\hat{\Gamma}_{i}(w,t,\delta)>0$ in the third line and fourth lines. 
		
		\noindent Solving the inequality above in $ \alpha(t,\delta)$, gives 
		\begin{align*}
		\alpha(t, \delta) &\le \frac{\sqrt{2K\Var(X_i- \langle w, \bm{X} \rangle)+ 28K\norm{w-e_i}_1(\langle w, \bm{\mu} \rangle - \mu_i ) }  - \sqrt{2K\Var(X_i- \langle w, \bm{X} \rangle)}~ }{14K\norm{w-e_i}_1}\\
		&= \frac{2\left( \langle w, \bm{\mu} \rangle - \mu_i \right)}{ \sqrt{2K\Var(X_i, \langle w, \bm{X} \rangle)+ 28K\norm{w-e_i}_1(\langle w, \bm{\mu} \rangle - \mu_i ) }  + \sqrt{2K\Var(X_i, \langle w, \bm{X} \rangle)}}.
		\end{align*}
		Therefore, we have
		\begin{align*}
		t &\ge K\log(\delta_t^{-1}) \left( \frac{\Var(X_i- \langle w, \bm{X} \rangle)}{(\langle w, \bm{\mu} \rangle - \mu_i )^2} + \frac{7\norm{w-e_i}_1}{\langle w, \bm{\mu} \rangle - \mu_i } \right)\\
		&\ge K\log(\delta_t^{-1})~\Xi_{i}(w).
		\end{align*}
		Which gives the first result.

		\noindent Now let us prove the second claim. Suppose that $\hat{\Gamma}_{i}(w,t,\delta) \le 0$. We have
		\begin{align*}
		\langle w, \bm{\mu} \rangle - \mu_i &= \hat{\Gamma}_{i}(w,t,\delta) - (\langle w, \hat{\bm{\mu}}_{t} \rangle - \hat{\mu}_{i,t} ) + \langle w, \bm{\mu} \rangle - \mu_i
		+ 2\sqrt{2K\widehat{V}_{t}(X_i, \langle w, \bm{X} \rangle)}~ \alpha(t,\delta)\\
        &\qquad \qquad + 14K\norm{w-e_i}_1 \alpha^2(t,\delta)\\
		&\le \hat{\Gamma}_{i}(w,t,\delta) +  \sqrt{2K\Var(X_i- \langle w, \bm{X} \rangle)} \alpha(t,\delta)+ 2\sqrt{2K\widehat{V}_{t}(X_i- \langle w, \bm{X} \rangle)} \alpha(t,\delta) \\
        &\qquad \qquad + \frac{42K}{3}\norm{w-e_i}_1 \alpha^2(t,\delta)\\
		&\le \hat{\Gamma}_{i}(w,t,\delta) +  3\sqrt{2K\Var(X_i- \langle w, \bm{X} \rangle)} \alpha(t,\delta)+ 9K \alpha^2(t,\delta) + \frac{42K}{3}\norm{w-e_i}_1 \alpha^2(t,\delta)\\
		&\le 3\sqrt{2K\Var(X_i- \langle w, \bm{X} \rangle)}\alpha(t,\delta)  +25K\norm{w-e_i}_1 \alpha^2(t,\delta) ,
		\end{align*}
		where we used \eqref{conc1} in the second line and \eqref{conc2} with $\hat{\Gamma}_{i}(w,t,\delta_t)\le0$ in the third line.
		Suppose that $\langle w, \bm{\mu} \rangle > \mu_i$. Solving the inequality above in $ \alpha(t,\delta)$, gives 
		\begin{align*}
		\alpha(t, \delta) &\ge \frac{\sqrt{18K\Var(X_i- \langle w, \bm{X} \rangle)+ 100K\norm{w-e_i}_1 (\langle w, \bm{\mu} \rangle - \mu_i ) }  - 3\sqrt{2K\Var(X_i- \langle w, \bm{X} \rangle)} }{50K\norm{w-e_i}_1}\\
		&= \frac{2(\langle w, \bm{\mu} \rangle - \mu_i)}{\sqrt{18K\Var(X_i- \langle w, \bm{X} \rangle)+ 100K\norm{w-e_i}_1(\langle w, \bm{\mu} \rangle - \mu_i ) }  + 3\sqrt{2K\Var(X_i- \langle w, \bm{X} \rangle)}}.
		\end{align*}
		Therefore, we have
		\begin{align*}
		t &\le K\log(\delta_t^{-1}) \left( \frac{18\Var(X_i- \langle w, \bm{X} \rangle)}{(\langle w, \bm{\mu} \rangle - \mu_i )^2} + \frac{50\norm{w-e_i}_1}{\langle w, \bm{\mu} \rangle - \mu_i } \right)\\
		&\le 36K\log(\delta_t^{-1}) \Xi_{i}(w).
		\end{align*}
		
		\noindent If $\langle w, \bm{\mu} \rangle \le \mu_i$, then $\Xi_{i}(w) = + \infty$ and the inequality above is straightforward.

	\end{proof}
	
	\begin{lemma}\label{lem:ultram}
		Let $i,j \text{ and } k \in \intr{K}$, we have:
		\begin{align*}
		\Lambda'_{ij} \le \max\{ \Lambda'_{ik}, \Lambda'_{kj}\}\\
		\Lambda_{ij} \le \max\{ \Lambda_{ik}, \Lambda_{kj}\}
		\end{align*}
	\end{lemma}
	\begin{proof}
		Let $i,j \text{ and } k \in \intr{K}$. 
		\noindent Suppose that $\mu_j \le \mu_i$. Hence, for any $k \in \intr{K}$, $ \mu_k \le \mu_i$ or $\mu_k \ge \mu_j$. Therefore 
		\[
		\max\{\Lambda'_{ik}; \Lambda'_{kj} \} = +\infty,
		\]
		which proves the result. The same argument applies to the quantities $\Lambda_{ik}$ and $\Lambda_{kj}$.

		Now suppose that $\mu_j > \mu_i$ (hence $\Lambda'_{ij} <+\infty$ and $\Lambda_{ij} <+\infty$). Let us start by proving the first claim. Recall the definition in Section~\ref{sec:not}:          
		\[\Lambda'_{ij} := \frac{V_{ij}}{(\mu_j-\mu_i)^2}.\]
		We have
		\begin{align*}
		\frac{\sqrt{V_{ij}}}{\mu_j - \mu_i} &\le \frac{\sqrt{V_{ik}}+\sqrt{V_{kj}}}{(\mu_j - \mu_k)+(\mu_k - \mu_i)}\\
		&\le \max\left \lbrace \frac{\sqrt{V_{kj}}}{\mu_j - \mu_k}, \frac{\sqrt{V_{ik}}}{\mu_k - \mu_i} \right \rbrace,
		\end{align*}
		where the first line follows by the triangle inequality and the second is a consequence of the inequality $\frac{a+b}{c+d} \le \max\{\frac{a}{c}, \frac{b}{d} \}$ (Lemma~\ref{cl:calpure_2}), which proves the first claim.

		Moreover, we have using the result above: for any $k$ such that $\mu_i <\mu_k <\mu_j$
		\begin{align*}
		\Lambda_{ij} &= \frac{V_{ij}}{(\mu_j-\mu_i)^2}+\frac{3}{\mu_j-\mu_i}\\
		&\le \max\left\lbrace \frac{V_{ik}}{(\mu_k-\mu_i)^2}+\frac{3}{\mu_j-\mu_i}, \frac{V_{kj}}{(\mu_j-\mu_k)^2}+\frac{3}{\mu_j-\mu_i} \right\rbrace\\
		&\le \max\left\lbrace \frac{V_{ik}}{(\mu_k-\mu_i)^2}+\frac{3}{\mu_k-\mu_i}, \frac{V_{kj}}{(\mu_j-\mu_k)^2}+\frac{3}{\mu_j-\mu_k} \right\rbrace\\
		&=\max\{\Lambda_{ik}, \Lambda_{kj}\}.
		\end{align*}
		If $k$ is such that $\mu_k <\mu_i$ or $\mu_k>\mu_j$, we have $\max\{\Lambda_{ik}, \Lambda_{kj}\} = + \infty$, which proves the statement.

	\end{proof}

	\section{Proof of Theorems~\ref{thm:2}}
	
	The proof of Theorem~\ref{thm:2} follows the same steps as the proof of Theorem~\ref{thm:0}.
	
	For any $i \in \intr{K}\setminus \{i^*\}$, let us define $\Upsilon_i$ by
	\begin{equation}\label{eq:def_si2}
	\Upsilon_i := \argmin_{j \in \intr{K}} \Lambda_{ij}.
	\end{equation}
	
	\begin{lemma}\label{lem:inter_1}
		Consider Algorithm~\ref{algo:1} with inputs $\delta \in (0,1)$. If $(\mathcal{B}_1)$ defined in \eqref{eq:eventa1} holds, then for any $i\in \intr{K} \setminus \{ i^*\}$ and $t \ge 1$: 
		
		If $i \in S_t$, then $\Upsilon_i \cap C_t \neq \emptyset$,
		where $\Upsilon_i$ is defined in \eqref{eq:def_si2}.
	\end{lemma}
	
	\begin{proof}
		Suppose that $(\mathcal{B}_1)$ holds. Let $t \ge 1$, $i\in \intr{K}\setminus\{i^*\}$. Proceeding by proof via contradiction, suppose that $\Upsilon_i \cap C_t = \emptyset$. This implies in particular that all elements in $\Upsilon_i$ were eliminated prior to $t$. Let $j$ denote the element of $\Upsilon_i$ with the largest mean:
		\[
		j \in \argmax_{l \in \Upsilon_i} \{\mu_l \}.
		\]
		Let $s$ denote the round where $j$ has failed the test (i.e.  $\exists k \in C_s, \hat{\Delta}_{jk}(s, \delta) >0$). 
		
		\noindent Hence, using Lemma~\ref{lem:ord_2}, we have
		\begin{equation}\label{eq:ineq_1}
		\frac{1}{4}\log(\delta_s^{-1}) \Lambda_{jk} \le s.
		\end{equation}
		Moreover,  $j$ was kept for testing up to round $82~s$ (i.e. $j \in C_{82s}$) and $82s <t$ (since $j \notin C_t$). At  round $82s$ we necessarily had $\hat{\Delta}_{ij}(82s, \delta) \le 0$. 
		
		\noindent Therefore, using Lemma~\ref{lem:ord_2}
		\begin{equation}\label{eq:ineq_2}
		82s \le 25 \log(\delta^{-1}_{82s})\Lambda_{ij}.
		\end{equation}
		Combining \eqref{eq:ineq_1} and \eqref{eq:ineq_2} gives
		\begin{equation*}
		\frac{82}{4}\log(\delta_s^{-1}) \Lambda_{jk} \le 25\log(\delta_{82s}^{-1}) \Lambda_{ij}.
		\end{equation*}
		
		\noindent Therefore
		\begin{equation}\label{eq:contr}
		\Lambda_{jk} \le \frac{50}{82}\left(1+\frac{\log(82)}{\log(1/\delta_s)} \right) \Lambda_{ij} \le \Lambda_{ij},
		\end{equation}
		recall that $\delta<1/4$, $s\ge 3$, $K\ge 3$, therefore: $\frac{50}{82}\left(1+\frac{\log(82)}{\log(1/\delta_s)} \right) \le 1$. 
		Using Lemma~\ref{lem:ultram}, we have
		\begin{equation}\label{eq:ultram1}
		\Lambda_{ik} \le \max\{\Lambda_{ij}, \Lambda_{jk}\}.
		\end{equation}
		We plug the bound $\Lambda_{jk} \le \Lambda_{ij}$ from \eqref{eq:contr} into \eqref{eq:ultram1} and obtain $\Lambda_{ik} \le \Lambda_{ij}$.
		Therefore $k\in \Upsilon_i$. 
		
		To conclude, recall that $k$ eliminates $j$, hence $\mu_k > \mu_j$. The contradiction arises from $k \in \Upsilon_i$ and  the definition of $j$ as the element with largest mean in $\Upsilon_i$.
	\end{proof}

	We introduce the following notation. For $i \in  \intr{K}$ and $t \ge 1$ let $N_{i,t}$ denote the number of queries made for arm $i$ up to round $t$
	\begin{equation}\label{eq:def_nit}
	N_{i,t} := \sum_{s=1}^{t} \mathds{1}\left(i \in C_s\right).
	\end{equation}

	\begin{lemma}\label{lem:ni0}
		Consider Algorithm~\ref{algo:1} with inputs $\delta \in (0,1)$. If $(\mathcal{B}_1)$ defined in \eqref{eq:eventa1} holds, then we have for each $i\in \intr{K} \setminus \{i^*\}$:
		\[
		\forall t \ge 1: \quad N_{i,t} \le 1368 \min_{j\in \intr{K}}\{\Lambda_{ij} \}   \log(216K^2\delta^{-1}\min_{j\in \intr{K}}\{\Lambda_{ij} \}),
		\] 
		where $N_{i,t}$ is defined in \eqref{eq:def_nit}.
	\end{lemma}
	\begin{proof}
		Suppose $(\mathcal{B}_1)$ holds. Let $i \in \intr{K}\setminus \{i^*\}$ and $t\ge 1$. Let $u$ denote the last round such that $i \in S_u$. Lemma~\ref{lem:inter_1} states that $\Upsilon_i \cap C_u \neq \emptyset$, where $\Upsilon_i$ is defined in \eqref{eq:def_si1}. Let $j \in \Upsilon_i \cap C_u$, since $i \in S_u$, we necessarily have
		\[
		\hat{\Delta}_{ij}(u-1, \delta) \le 0.
		\]
		Using Lemma~\ref{lem:ord_2}, we have
		\[
		u-1 \le \frac{25}{2} \log(\delta_{u-1}^{-1}) \Lambda_{ij}.
		\]
		Recall that $u$ is the last round such that $i \in S_u$, hence $i \notin C_{19u+1}$. Therefore, for any $t \ge 1$
		\begin{align*}
		N_{i,t} = 82u &\le 1013\log(\delta_{u}^{-1})  \Lambda_{ij}\\
		&\le 2026 \min_{j\in \intr{K}}\{\Lambda_{ij} \}\log(216K^2\delta^{-1}\min_{j\in \intr{K}}\{\Lambda_{ij} \}),
		\end{align*}
		where we used Lemma~\ref{cl:calpure2} with $x=u$ and $c=\delta/(2K^2)$.
	\end{proof}
	
	\paragraph{Proof for Theorem~\ref{thm:2}}
	
	Suppose Assumptions~\ref{a:0} and~\ref{a:bounded} hold. Consider Algorithm~\ref{algo:1} with input $\delta \in (0,1/4)$. Suppose that event $(\mathcal{B}_1)$ holds. 
	
	\noindent We have by definition of the total number of queries made $N$:
	\[
	N = \sum_{i \in \intr{K}\setminus \{i^*\}} N_{i,t}.
	\] 
	\noindent Therefore, Lemma~\ref{lem:ni0} gives the result.
	For the second result, consider Algorithm~\ref{algo:1} without line~\eqref{line:algo} (we stop sampling arms directly after their elimination from $S_t$). Lemma~\ref{lem:elim_2} guarantees with probability at least $1-\delta$, the optimal arm $i^*$ always belongs to $S_t$. Therefore, Lemma~\ref{lem:ord_1} guarantees that after at most $\frac{25}{2} \log(\delta_t^{-1}) \Lambda_{i i^*}$ rounds, we have $\hat{\Delta}_{ii^*}(t,\delta) >0$, which leads to the elimination of $i$.

	\section{Proof of Theorem~\ref{thm:0}}

	For any $i \in \intr{K}\setminus \{i^*\}$, let us define $\Upsilon'_i$ by
	\begin{equation}\label{eq:def_si2g}
	\Upsilon'_i := \argmin_{j \in \intr{K}} \left\lbrace  \Lambda'_{ij} \vee \frac{1}{4} \right\rbrace.
	\end{equation}
	
	\begin{lemma}\label{lem:inter_1g}
		Consider Algorithm~\ref{algo:1} with inputs $\delta \in (0,1)$. If $(\mathcal{A}_1)$ defined in \eqref{eq:eventa1} holds, then for any $i\in \intr{K} \setminus \{ i^*\}$ and $t \ge 1$: 
		
		If $i \in S_t$, then $\Upsilon'_i \cap C_t \neq \emptyset$,
		where $\Upsilon'_i$ is defined in \eqref{eq:def_si2g}.
	\end{lemma}
	
	\begin{proof}
		Suppose that $(\mathcal{A}_1)$ holds. Let $t \ge 2$, $i\in \intr{K}\setminus\{i^*\}$. Proceeding by proof via contradiction, suppose that $\Upsilon'_i \cap C_t = \emptyset$. This implies in particular that all elements in $\Upsilon'_i$ were eliminated prior to $t$. Let $j$ denote the element of $\Upsilon'_i$ with the largest mean:
		\[
		j \in \argmax_{l \in \Upsilon'_i} \{\mu_l \}.
		\]
		Let $s$ denote the round where $j$ has failed the test (i.e.  $\exists k \in C_s, \hat{\Delta}'_{jk}(s, \delta) >0$). 
		
		\noindent Hence, using Lemma~\ref{lem:ord_0}, we have
		\begin{equation}\label{eq:ineq_1g}
		\frac{1}{2}\Lambda'_{jk} \log(1/\delta_s)  \le s-1.
		\end{equation}
		We consider two cases:
		\paragraph{Case 1: $\Lambda'_{ij}\ge \frac{1}{5}$.}
		Observe that $j$ was kept for testing up to round $35s$ (i.e. $j \in C_{82s}$) and $82s <t$ (since $j \notin C_t$). At  round $82s$ we necessarily had $\hat{\Delta}'_{ij}(82s, \delta) \le 0$. 
		
		\noindent Therefore, using Lemma~\ref{lem:ord_0}
		\begin{equation}\label{eq:ineq_2g}
		82s-1 \le 25 \Lambda'_{ij} \log(1/\delta_{35s}).
		\end{equation}
		Combining \eqref{eq:ineq_1g} and \eqref{eq:ineq_2g} gives
		\begin{equation*}
		\frac{82}{2} \Lambda'_{jk}\log(1/\delta_s)  \le 25\Lambda'_{ij}\log(1/\delta_{82s}).
		\end{equation*}
		Therefore
		\begin{equation}\label{eq:contrg}
		\Lambda'_{jk} \le \frac{50}{82} \left(1 + \frac{\log(82)}{\log(\delta_s^{-1})}\right) \Lambda'_{ij} \le \Lambda'_{ij}.
		\end{equation}
		Using Lemma~\ref{lem:ultram}, we have
		\begin{equation}\label{eq:ultram1g}
		\Lambda'_{ik} \le \max\{\Lambda'_{ij}, \Lambda'_{jk}\}.
		\end{equation}
		We plug the bound $\Lambda'_{jk} \le \Lambda'_{ij}$ from \eqref{eq:contrg} into \eqref{eq:ultram1g} and obtain $\Lambda'_{ik} \le \Lambda'_{ij}$.
		Therefore $k\in \Upsilon'_i$. 
		
		To conclude, recall that $k$ eliminates $j$, hence $\mu_k > \mu_j$. The contradiction arises from $k \in \Upsilon'_i$ and  the definition of $j$ as the element with largest mean in $\Upsilon'_i$.
		
		\paragraph{Case 2: $\Lambda'_{ij}< \frac{1}{5}$.} As in the previous case, observe that $j$ was kept for testing up to round $82s$ (i.e. $j \in C_{82s}$) and $82s <t$ (since $j \notin C_t$). At  round $82s$ we necessarily had $\hat{\Delta}_{ij}(82s, \delta) \le 0$. 
		
		\noindent Therefore, using Lemma~\ref{lem:ord_0}
		\begin{equation}\label{eq:ineq_2g2}
		82s-1 \le \frac{3}{2\log\left(1/(3\sqrt{2}\Lambda'_{ij})\right)} \log(1/\delta_{82s}).
		\end{equation}
		Combining \eqref{eq:ineq_1g} and \eqref{eq:ineq_2g2} gives
		\begin{equation*}
		\frac{82}{2} \Lambda'_{jk}\log(1/\delta_s)  \le  \frac{3\log(1/\delta_{82s})}{2\log(\sqrt{2}/(5\Lambda'_{ij}))}.
		\end{equation*}
		Therefore
		\begin{align*}
		\Lambda'_{jk} &\le \frac{3}{82} \left(1+\frac{\log(82)}{\log(1/\delta_s)}\right)~\frac{1}{\log(\sqrt{2}/(5\Lambda'_{ij}))}\\
		&\le \frac{6}{82\log(1/(3\sqrt{2}\Lambda'_{ij}))}.
		\end{align*}
		Recall that using Lemma~\ref{lem:ultram}, we have: $\Lambda'_{ik} \le \max\{\Lambda'_{ij}, \Lambda'_{jk}\}$. This implies necessarily that $\Lambda'_{ik} \le \Lambda'_{jk}$. Otherwise if the maximum of the l.h.s is $\Lambda'_{ij}$, we get $\Lambda'_{ij} \le \Lambda'_{ij}$ then $k \in \Upsilon'_i$, therefore by definition of $j$ as the largest mean of $\Upsilon'_i$: $\mu_k \le \mu_j$, which contradicts the fact that $k$ eliminated $j$. We conclude that (using $\Lambda'_{ij}<1/5$):
		\[
		\Lambda'_{ik} \le \frac{6}{82\log(\sqrt{2}/(5\Lambda'_{ij}))}< \frac{1}{4}.
		\]
		Therefore $k \in \argmin_{j \in \intr{K}} \left\lbrace\Lambda'_{ij} \vee \frac{1}{4}\right\rbrace$, which similarly to the case above, leads to a contradiction with the definition of $j$.

	\end{proof}

	We introduce the following notation. For $i \in  \intr{K}$ and $t \ge 1$ let $N_{i,t}$ denote the number of queries made for arm $i$ up to round $t$
	\begin{equation}\label{eq:def_nitg}
	N_{i,t} := \sum_{s=1}^{t} \mathds{1}\left(i \in C_s\right).
	\end{equation}

	\begin{lemma}\label{lem:ni0g}
		Consider Algorithm~\ref{algo:1} with inputs $\delta \in (0,1)$. If $(\mathcal{A}_1)$ defined in \eqref{eq:eventa2-G} holds, then we have for each $i\in \intr{K} \setminus \{i^*\}$:
		\[
		\forall t \ge 1: \quad N_{i,t} \le 4100\log(216K\delta^{-1} \min_{j\in \intr{K}}\{ \Lambda'_{ij}\vee 1\}) ) \min_{j\in \intr{K}}\left\lbrace \Lambda'_{ij}\vee \frac{1}{4}\right\rbrace,
		\] 
		where $N_{i,t}$ is defined in \eqref{eq:def_nit}.
	\end{lemma}
	\begin{proof}
		Suppose $(\mathcal{A}_1)$ holds. Let $i \in \intr{K}\setminus \{i^*\}$ and $t\ge 1$. Let $u$ denote the last round such that $i \in S_u$. Lemma~\ref{lem:inter_1g} states that $\Upsilon'_i \cap C_u \neq \emptyset$, where $\Upsilon'_i$ is defined in \eqref{eq:def_si1}. Let $j \in \Upsilon'_i \cap C_u$, since $i \in S_u$, we necessarily have
		\[
		\hat{\Delta}'_{ij}(u-1, \delta) \le 0.
		\]
		Using Lemma~\ref{lem:ord_0}, we have
		\[
		u-1 \le \frac{25}{\log\left(1+1/\Lambda'_{ij}\right)} \log(\delta_{u-1}^{-1}).
		\]
		Recall that $u$ is the last round such that $i \in S_u$, hence $i \notin C_{35u+1}$. Therefore, for any $t \ge 1$
		\begin{align*}
		N_{i,t} = 82u &\le \frac{2050}{\log\left(1+1/\Lambda'_{ij}\right)}\log(\delta_{u}^{-1})\\
		&\le 4100 \log(216K(\Lambda'_{ij}+1)\delta^{-1}) \frac{1}{\log\left(1+1/\Lambda'_{ij}\right)}\\
		&\le 4100 \log(216K(\Lambda'_{ij}+1)\delta^{-1}) \frac{1}{\log\left(1+1/\Lambda'_{ij}\right)}\\
		&\le 4100 \log(216K(\Lambda'_{ij}+1)\delta^{-1}) \frac{1}{\log\left(1+1/(\Lambda'_{ij} \vee 0.25)\right)}\\
		&\le 4100 \log(216K(\Lambda'_{ij}+1)\delta^{-1}) \frac{8}{7}(\Lambda'_{ij} \vee 0.25)\\
		&\le 4100 \log(216K(\Lambda'_{ij}+1)\delta^{-1}) \min_{j\in \intr{K}} (\Lambda'_{ij} \vee 0.25),
		\end{align*}
		$j \in \Upsilon_i$ and Lemma~\ref{cl:calpure2} with $x=u$ and $c=\delta/(2K^2)$.
	\end{proof}
	
	\paragraph{Proof for Theorem~\ref{thm:0}}
	
	Suppose Assumptions~\ref{a:0} and~\ref{a:gaussian}. Consider Algorithm~\ref{algo:1} with input $\delta \in (0,1/4)$. Suppose that event $(\mathcal{B}_1)$ holds. 
	
	\noindent We have by definition of the total number of queries made $N$:
	\[
	N = \sum_{i \in \intr{K}\setminus \{i^*\}} N_{i,t}.
	\] 
	\noindent Therefore, Lemma~\ref{lem:ni0} gives the result.
	For the second result, consider Algorithm~\ref{algo:1} without line~\eqref{line:algo} (we stop sampling arms directly after their elimination from $S_t$). Lemma~\ref{lem:elim_2} guarantees with probability at least $1-\delta$, the optimal arm $i^*$ always belongs to $S_t$. Therefore, Lemma~\ref{lem:ord_1} guarantees that after at most $\frac{25}{2} \log(\delta_t^{-1}) \frac{1}{\log\left(1+1/\Lambda'_{i i^*}\right)}$ rounds, we have $\hat{\Delta}'_{ii^*}(t,\delta) >0$, which leads to the elimination of $i$.

	\section{Proof of Theorem~\ref{thm:1}}

	We provide the same type of guarantees for Algorithm~\ref{algo:0}. 
	
	\noindent For any $u, v \in \mathbb{B}_1^+$, we overload the notation $\Xi_i(u)$ into
	\[
	\Xi_{u}(v) := \left\{
	\begin{array}{ll}
	+\infty & \mbox{if } \langle u, \bm{\mu} \rangle \le \langle v, \bm{\mu} \rangle \\
	\max\left\lbrace \frac{\Var( \langle \bm{X}, u \rangle- \langle \bm{X}, v \rangle )}{(\langle u, \bm{\mu} \rangle - \langle v, \bm{\mu} \rangle )^2} , \frac{ 3\norm{u-v}_1}{\langle v, \bm{\mu} \rangle- \langle u, \bm{\mu} \rangle } \right\rbrace  & \mbox{otherwise }
	\end{array}
	\right.
	\]
	In particular we have $\Xi_{e_i}(w) = \Xi_i(w)$, where $(e_i)_{i \in \intr{K}}$ is the canonical basis of $\R^K$.
	We say that an arm $i\in \intr{K}$ has failed the $\Gamma$-test at round $t$, if
	\[
	\sup_{w \in \mathbb{B}_1^+(C_t\setminus \{i\})} \hat{\Gamma}_i(w, t, \delta) >0. 
	\]
	\begin{lemma}\label{lem:ultram2}
		Let $i \in \intr{K}$, $u,v \in \mathbb{B}_1^+(\intr{K}\setminus \{i\})$, we have
		\[
		\Xi_{i}(v) \le \max\{ \Xi_{i}(u), \Xi_u(v) \}.
		\]
	\end{lemma}
	\begin{proof}
		Let $i \in \intr{K}$ and $u,v \in \mathbb{B}_1^+(\intr{K}\setminus \{i\})$. 
		Suppose that $\mu_i < \langle u, \bm{\mu} \rangle < \langle v, \bm{\mu} \rangle$. We have
		\begin{align*}
		\frac{\sqrt{\Var(X_i- \langle v, \bm{X}\rangle)}}{\langle v, \bm{\mu}\rangle - \mu_i} &\le \frac{ \sqrt{\Var(X_i- \langle u, \bm{X}\rangle)} + \sqrt{\Var(\langle u, \bm{X}\rangle- \langle v, \bm{X}\rangle)} }{(\langle v, \bm{\mu}\rangle - \langle u, \bm{\mu}\rangle)+(\langle u, \bm{\mu}\rangle - \mu_i)}\\
		&\le \max\left \lbrace \frac{\sqrt{\Var(X_i- \langle u, \bm{X}\rangle)}}{\langle v, \bm{\mu}\rangle - \langle u, \bm{\mu}\rangle}, \frac{\sqrt{\Var(\langle u, \bm{X}\rangle- \langle v, \bm{X}\rangle)}}{\langle u, \bm{\mu}\rangle - \mu_i} \right \rbrace,
		\end{align*}
		where the first line follows by the triangle inequality and the second is a consequence of the inequality $\frac{a+b}{c+d} \le \max\{\frac{a}{c}, \frac{b}{d} \}$ for positive numbers (Lemma~\ref{cl:calpure_2}). 
		Moreover we have
		\begin{align*}
			\frac{\norm{v-e_i}_1}{\langle v-e_i, \bm{\mu}\rangle} &\le \frac{\norm{v-u}_1 + \norm{u-e_i}_1}{\langle v-u, \bm{\mu}\rangle +\langle u-e_i, \bm{\mu}\rangle}\\
			&\le \max\left\lbrace \frac{\norm{v-u}_1}{\langle v-u, \bm{\mu}\rangle }, \frac{ \norm{u-e_i}_1}{\langle u-e_i, \bm{\mu}\rangle} \right\rbrace,
		\end{align*}
		where we used the inequality $\frac{a+b}{c+d} \le \max\{\frac{a}{c}, \frac{b}{d} \}$ for positive numbers (Lemma~\ref{cl:calpure_2}).

		\noindent Combining the previous bounds, we obtain the result.
		
		\noindent If  $\mu_i \ge \langle u, \bm{\mu} \rangle $ or $ \langle u, \bm{\mu} \rangle \ge \langle v, \bm{\mu} \rangle $. We have
		\[
		\max\{ \Xi_i(u); \Xi_u(v) \} = +\infty,
		\]
		which proves the result.
	\end{proof}
	
	For any $i \in \intr{K}\setminus \{i^*\}$, let us define $\Psi_i$ by
	\begin{equation}\label{eq:def_si1}
	\Psi_i := \argmin_{w \in \mathbb{B}_1^+(\intr{K}\setminus \{i\})} \Xi_{i}(w).
	\end{equation}

	\begin{lemma}\label{lem:inter_2}
		Consider Algorithm~\ref{algo:0} with input $\delta\in (0,1)$. If $(\mathcal{B}_2)$ defined in \eqref{eq:eventa1} holds, then for any $i\in \intr{K}\setminus \{i^*\}$, $t\ge 1$:
		
		If $i\in S_t$, then there exists a vector $w^* \in \Psi_i$ such that:  $\text{supp}(w^*) \subseteq C_t$.

	\end{lemma}
	\begin{proof}
		Let $t \ge 1$, $i \in \intr{K}\setminus \{i^*\}$. We take $w^*$ to be one of the vectors from the set $\Psi_{i}$, such that its support was jointly queried the most up to round $t$. 
		More formally: 
		\[
		w^* \in \argmax_{w \in \Psi_i} \{\langle w, \bm{\mu}\rangle\}.
		\]  
		
		\noindent Proceeding by proof via contradiction, we suppose that $ \text{supp}(w^*) \not\subset C_t$. Then, we will build a vector $w' \in \mathbb{B}_1^+$, such that $\langle w^*, \bm{\mu} \rangle < \langle w', \bm{\mu} \rangle$, the contradiction follows from the definition of $w^*$. Let $j$ be the first eliminated element in $\text{supp}(w^*)$. Let $s$ denote the round where $j$ has failed the $\Gamma$-test (i.e. $\exists \tilde{w} \in \mathbb{B}_1^+, \hat{\Gamma}_{j}(\tilde{w}, s, \delta) >0$).
		
		\noindent Let us define $w' \in \mathbb{R}^K$ as follows: $w'_j = w^*_j \tilde{w}_j$ and for $k \in \intr{K}\setminus \{j\}$, $w'_k = w^*_{k} + w^*_{j} \tilde{w}_k$. Recall that
		\begin{align*}
		\norm {w'}_1 &= \sum_{k \in \intr{K} \setminus \{j\}} w^*_{k} + \sum_{k \in \intr{K}} w^*_{j} \tilde{w}_k\\
		&= 1 - w^*_{j} + w^*_{j} \norm{\tilde{w}}_1\\
		&= 1,
		\end{align*}
		where we used the fact that $\norm{\tilde{w}}_1=1$. We conclude that $w' \in \mathbb{B}_1^+$.
		
		Let us show that $w' \in \Psi_{i}$.
		We have
		\begin{align}
		\langle w^* - w', \bm{\mu} \rangle &= w^*_{j}(1-\tilde{w}_j) \mu_j + \sum_{k \in \intr{K} \setminus \{j\}} (w^*_{k} - w^*_{k} - w^*_{j} \tilde{w}_k)\mu_k \notag\\
		&= w^*_{j} \mu_j -w^*_{j} \tilde{w}_j \mu_j - w^*_{j}\sum_{k \in \intr{K} \setminus \{j\}}\tilde{w}_k u_k \notag\\
		&= w^*_{j} (\mu_j - \langle \tilde{w}, \bm{\mu} \rangle).\label{eq:diff}
		\end{align}
		Moreover
		\begin{align*}
			\norm{w^* - w'}_1 &= \sum_{k =1}^{K} \lvert w^*_k-w'_k \rvert\\
			&= w^*_j \lvert 1-\tilde{w}_j \rvert +w^*_j \sum_{k \in \intr{K}\setminus \{j\}}\tilde{w}_k\\
			&= w^*_j \norm{\tilde{w}-e_j}_1.   
		\end{align*}
		
		\noindent Using \eqref{eq:diff} we have
		\begin{align*}
		\Xi_{w^*}(w') &= \max\left\lbrace \frac{\Var\left( \langle w^*-w', \bm{X} \rangle \right)}{(\langle w^* - w', \bm{\mu} \rangle )^2}; \frac{ \norm{w^*-w'}_1}{\langle w' - w^*, \bm{\mu} \rangle} \right\rbrace\\
		&= \max\left\lbrace\frac{\Var\left( w_{j} (X_j - \langle \tilde{w}, \bm{X} \rangle) \right)}{w_{j}^2 (\mu_j - \langle \tilde{w}, \bm{u} \rangle)^2}; \frac{\norm{\tilde{w}-e_j}_1}{\langle \tilde{w}, \bm{\mu} \rangle -\mu_j }\right\rbrace\\
		&= \Xi_j(\tilde{w}).
		\end{align*}
		Therefore, using Lemma~\ref{lem:ultram2}
		\begin{align}\label{eq:ineq_3}
		\Xi_{i}(w') &\le \max\left\lbrace \Xi_{i}(w^*); \Xi_{w^*}(w') \right\rbrace \notag \\
		&= \max\left\lbrace \Xi_{i}(w^*); \Xi_j(\tilde{w}) \right\rbrace.
		\end{align}
		
%		In particular, for $u = \bm{\mu}$, we have
%		\begin{equation}\label{eq:contr_1}
%		\langle w^* - w', \bm{\mu} \rangle =  w^*_{j} (u_j - \langle \tilde{w}, u \rangle) <0,
%		\end{equation}
%		since $\tilde{w}$ eliminated $j$.
		
		\noindent Recall that $\hat{\Gamma}_{j}(\tilde{w},s, \delta) >0$. Hence using Lemma~\ref{lem:ord_1}, we have
		\begin{equation}\label{eq:ineq_4}
		K \log(2K^2\delta_s^{-1})\Xi_{j}(\tilde{w}) \le s.
		\end{equation}
		
		\noindent Moreover, since $j$ failed the $\Gamma$-test at round $s$, we have by construction of Algorithm~\ref{algo:0} $j \in C_{98s}$. Recall that $j$ is the first element of the support of $w^*$ that was eliminated, then we necessarily have $\text{supp}(w^*) \subset C_{98s}$. Since we assumed that $\text{supp}(w^*) \not\subset C_{t}$, we have $98s <t$, hence $i \in C_{98s}$ and $\hat{\Gamma}_{i}(w^*,98s, \delta) \le 0$. Using Lemma~\ref{lem:ord_1}
		\begin{equation}\label{eq:ineq_5}
		98s \le 36 K \log(2K^2\delta_{98s}^{-1})~\Xi_i(w^*).
		\end{equation}
		Combining inequalities \eqref{eq:ineq_4} and \eqref{eq:ineq_5}, we have
		\[
		98K\log(2K^2\delta_{s}^{-1})\Xi_j(\tilde{w}) < 36K\log(2K^2\delta_{98s}^{-1}) \Xi_i(w^*).
		\]
		Therefore
		\begin{align*}
		\Xi_j(\tilde{w}) &\le \frac{36}{98} \frac{\log(2K^2\delta_{98s}^{-1})}{\log(2K^2\delta_{s}^{-1})} \Xi_i(w^*)\\
		&\le \frac{36}{98} \left( 1+ \frac{\log(98)}{\log(\delta_{s}^{-1})}\right) \Xi_i(w^*)\\
		&\le \Xi_i(w^*).
		\end{align*}
		Combining the bound above with \eqref{eq:ineq_3}, we conclude that $\Xi_i(w') \le \Xi_i(w^*)$. Hence $w' \in \Psi_i$.
		
		Finally, recall that by \eqref{eq:diff}  $\langle w', \bm{\mu} \rangle > \langle w^*, \bm{\mu} \rangle$ (since $\tilde{w}$ eliminated $j$: $\langle \tilde{w}, \mu_j \rangle -\mu_j >0$). The conclusion follows from $w' \in \Psi_i$ and the definition of $w^*$.
	\end{proof}

	\begin{lemma}\label{lem:ni1}
		Consider Algorithm~\ref{algo:0} with input $\delta \in (0,1)$. If $(\mathcal{B}_2)$ defined in \eqref{eq:eventa1} holds, then we have for each $i\in \intr{K}$, $t\ge 1$:
		\[
		N_{i,t} \le  10540 \log\left(1296K^2\Xi_i(w^*)\delta^{-1}\right) K\Xi_i(w^*).
		\] 
	\end{lemma}
	\begin{proof}
		Suppose $(\mathcal{A}_2)$ holds. Let $i \in \intr{K}\setminus \{i^*\}$ and $t\ge 1$. Let $u$ denote the last round such that $i \in S_u$. Lemma~\ref{lem:inter_2} states that there exists $w^* \in \Psi_i$ such that $\text{supp}(w^*) \subset C_u$, where $\Psi_i$ is defined in \eqref{eq:def_si2}. Since $i \in S_u$, we necessarily have:
		\[
		\hat{\Gamma}_{i}(w^*,u-1, \delta) \le 0.
		\]
		Using Lemma~\ref{lem:ord_2}, we have
		\[
		u-1 \le 108K\log(2K^2\delta_{u-1}^{-1}) \Xi_i(w^*).
		\]
		Recall that $u$ is the last round such that $i\in S_u$, therefore $i \notin C_{98u+1}$. Hence, for any $t\ge 1$
		\begin{align*}
		N_{i,t} = 98u &\le 10540K\log(2K^2\delta_{u}^{-1}) \Xi_i(w^*)\\
		&\le 10540 \log\left(1296K^2\Xi_i(w^*)\delta^{-1}\right) K\Xi_i(w^*),
		\end{align*}
		where we used Lemma~\ref{cl:calpure2} with $x=u$ and $c= \delta/(2K^2)$.
	\end{proof}
	
	The conclusion for the proof of Theorem~\ref{thm:1} is similar to the conclusion of the proof of Theorem~\ref{thm:2}.

	\section{Proof of Theorem~\ref{thm:lower_bound_b}}
	
	Without loss of generality assume that $ \mu_1>\mu_2 \ge \dots \ge \mu_K$. For any Bernoulli variables $(X_i)_i$ with means $(\mu_i)_i$. Fix a sequence of positive numbers $(V_{i1})_{i \ge 2}$ such that for each $i \ge 2$:
	\begin{equation}\label{eq:constraint}
	(\mu_1-\mu_i) - (\mu_1 - \mu_i)^2 \le \Var(X_1 - X_i) \le \min\left(2-(\mu_1+\mu_i);~\mu_1+\mu_i \right)-(\mu_1-\mu_i)^2.
	\end{equation}
	Below we build $K$ Bernoulli variables $(X_i)_{i \in \intr{K}}$ such that $\mathbb{E}[X_i] = \mu_i$, $\Var(X_i - X_1) \le V_{i1}$ and $1/4 \le \mu_i \le 3/4$, for all $i \in \intr{K}$. Recall that the set of Bernoulli variables satifying the last constraints is denoted $\mathbb{B}_K(\bm{\mu}, \bm{V})$.
	\paragraph{Building a distribution in $\mathbb{B}_K(\bm{\mu}, \bm{V})$:} 
	
	Let $(U_t)_{t\ge 1}$ and $(W_{i,t})_{i \in \intr{K}, t\ge 1}$ denote sequences of independent variables following the uniform distribution on the interval $[0,1]$. Let $(a_i,b_i)_{i\in \intr{K}}$ denote a sequence of numbers in $[0,1]$ to be specified later. We define the arms variables $(X_i)_{i \in \intr{K}}$ as follows:
	
	\begin{itemize}
		%\item $\forall i \in \{2, \dots, K\}$: $a_i+b_i \le 1$.
		\item $X_{1,t} = \mathds{1}\left(U_t \le \mu_1\right)$ for each $t\ge 1$.
		\item For $i\ge 2$, $t\ge 1$. We consider two cases:
		\begin{itemize}
			\item If $V_{i1} \le \mu_1+\mu_i - \mu_1^2 - \mu_i^2$: Let $X_{i,t} = \mathds{1}\left(\left\lbrace U_t \le a_i \right\rbrace~ \text{or}~\left\lbrace W_{i,t} \le b_i \right\rbrace\right)$.
			\item If $V_{i1} \ge \mu_1+\mu_i - \mu_1^2 - \mu_i^2$: Let $X_{i,t} = \mathds{1}\left( W_{i,t} \le \mu_i \right)$.
		\end{itemize} 
	\end{itemize}
	Now let us specify our choice for the sequences $(a_i)$ and $(b_i)$.
	\paragraph{If $V_{i1} \le \mu_1+\mu_i - \mu_1^2 - \mu_i^2$:}
	
	The first constraint is with respect to the means of $(X_i)$. We need to have for each $i \in \intr{K}$: $\mathbb{E}[X_i] = \mu_i$, this implies the following for each $i \ge 2$:
	\begin{align*}
	\mathbb{E}[X_i] &= \mathbb{E}[\mathds{1}\left(\left\lbrace U_t \le a_i \right\rbrace~ \text{or}~\left\lbrace W_{i,t} \le b_i \right\rbrace\right)]\\
	&= a_i + b_i -a_ib_i.
	\end{align*}
	We therefore have
	\begin{equation}\label{eq:const_1}
	a_i + b_i - a_ib_i = \mu_i
	\end{equation}
	The second constraint is with tespect to the variance of the variable $(X_1-X_i)$, we set $\Var(X_1-X_i) = V_{i1}$. This implies the following
	\begin{align*}
	\Var(X_1-X_i) &= \mathbb{E}\left[\left(\mathds{1}\left(U_t \le \mu_1\right) -\mathds{1}\left( \left\lbrace U_t \le a_i \right\rbrace~ \text{or}~\left\lbrace W_{i,t} \le b_i \right\rbrace \right)\right)^2\right] - (\mu_1 - \mu_i)^2\\
	&= \mu_1 + \mu_i -2 \left(a_i + (\mu_1-a_i) b_i\right)- (\mu_1 - \mu_i)^2.
	\end{align*}
	We therefore have that $a_i$ and $b_i$ satisfy
	\begin{equation}\label{eq:const_2}
	\mu_1 + \mu_i -2 \left(a_i + (\mu_1-a_i) b_i\right)- (\mu_1 - \mu_i)^2 = V_{i1}.
	\end{equation}
	Solving in $a_i$ and $b_i$ for the system:  \eqref{eq:const_1} and \eqref{eq:const_2} is:
	\begin{align*}
	b_i &= \frac{V_{i1} -\left( (\mu_1 - \mu_i)-(\mu_1 - \mu_i)^2\right) }{2(1-\mu_1)}\\
	a_i &= \frac{\mu_i - b_i}{1-b_i}
	\end{align*}
	Observe that when $V_{i1} \le \mu_1+\mu_i - \mu_1^2 - \mu_i^2$, we have :
	\begin{align*}
	b_i &= \frac{V_{i1} - (\mu_1-\mu_i) + (\mu_1-\mu_i)^2}{2(1-\mu_1)}\\
	&\le \frac{2\mu_i - 2\mu_i\mu_1}{2(1-\mu_1)} \\
	&= \mu_i.
	\end{align*}
	Therefore $b_i \le \mu_i$. Moreover, using Lemma~\ref{lem:ber}, we have $b_i \ge 0$. We conclude that $b_i \in [0,\mu_i]$, which implies that $a_i \in [0,1]$.

	\paragraph{If $V_{i1} \ge \mu_1+\mu_i - \mu_1^2 - \mu_i^2$:} 
	Let us prove in case that we have $\Var(X_1 - X_i) \le V_{i1}$. Recall that $X_i$ and $X_1$ are independent. Therefore:
	\begin{align*}
	\Var(X_1-X_i) &= \mu_1+\mu_i -2\mu_1\mu_i-(\mu_1-\mu_i)^2\\
	&\le V_{i1} +\mu_1^2+\mu_i^2-2\mu_i\mu_1 - (\mu_1-\mu_i)^2\\
	&= V_{i1}.
	\end{align*}
	As a conclusion we have in both cases: the distribution belongs to $\mathbb{B}_{K}(\bm{\mu}, \bm{\sigma})$.

	\paragraph{Developing the lower bound for $T_i$ in the case $V_{i1} \le \mu_1+\mu_i - \mu_1^2 - \mu_i^2$:}
	Given a joint distribtion of Bernoulli variables in $\mathbb{B}_K(\bm{\mu}, \bm{V})$, let us develop the corresponding lower bound.
	
	Let $\mathcal{A}$ be a $\delta$ sound strategy. 
	Denote $\mathbb{P}^{(1)}$ the joint distribution of $X_i$ defined above. For $i \in \{2, \dots, K\}$, denote by $\mathbb{P}^{(i)}$ the alternative probability distribution where only arm $i$ is modified as follows:
	
	\[
	X_{i,t} = \mathds{1}\left(\left\lbrace U_t \le a_i \right\rbrace~\text{or}~\left\lbrace W'_{i,t} \le \frac{\mu_1+\epsilon-\mu_i}{1-a_i} + b_i \right\rbrace\right),
	\]
	where $\epsilon>0$ and $(W'_{i,t})_{t\ge 1}$ is a sequence of independent variables following the uniform distribution in $[0,1]$. Observe that $\mathbb{E}^{(i)}[X_i] = \mu_1+\epsilon$, therefore under $\mathbb{P}^{(i)}$ the optimal arm is arm $i$.

	Fix $i \ge 2$, since $\mathcal{A}$ is $\delta$-sound, 
	\begin{equation}\label{eq:delta_sound2}
	\mathbb{P}^{(1)} (\psi \neq 1) + \mathbb{P}^{(i)}(\psi \neq i) \le 2\delta.
	\end{equation}

	\noindent Using Theorem 2.2 of \cite{tsybakov2003optimal}, we have
	\[
	\mathbb{P}^{(1)}(\psi \neq 1) + \mathbb{P}^{(i)}(\psi = 1) \ge 1 - \text{TV}\left(\mathbb{P}^{(1)}, \mathbb{P}^{(i)} \right),
	\]
	where $\text{TV}(\mathbb{P}, \mathbb{Q})$ denotes the total variation distance between $\mathbb{P}$ and $\mathbb{Q}$. Using \eqref{eq:delta_sound}, we conclude that
	\begin{equation}\label{eq:bound_tv1}
	\text{TV}(\mathbb{P}^{(1)}, \mathbb{P}^{(i)}) \ge 1-2\delta. 
	\end{equation}
	\noindent For a subset $A \subset \intr{K}$, denote $N_A$ the total number of rounds where the jointly queried arms are the elements of $A$, and let $\mathbb{P}_A^{(.)}$ denote the joint distribution of $(X_j)_{j \in A}$ under $\mathbb{P}^{(.)}$. Under Protocol~\ref{algo:gp}, the learner has to choose a subset $C_t \subset \intr{K}$ in each round $t$, and observes only the rewards of arms in $C_t$. Hence, we can apply Lemma~\ref{lem:tvkl} which bounds the total variation distance in terms of Kullback-Leibler discrepancy to this case where arms correspond to subsets of $\intr{K}$. This gives
	%\noindent For a subset $A \subset \intr{K}$, denote $N_A$ the total number of rounds where the jointly queried arms are the elements of $A$.
	%Therefore we have:
	\begin{align}
	\text{TV}\left(\mathbb{P}^{(1)}, \mathbb{P}^{(i)}\right) &\le 1 - \frac{1}{2} \exp\left\lbrace -\sum_{A \subset \intr{K}} \mathbb{E}^{(1)}[N_A] \text{KL}\left(\mathbb{P}_A^{(1)}, \mathbb{P}_A^{(i)}\right) \right\rbrace \nonumber\\
	&= 1 - \frac{1}{2} \exp\left\lbrace -\sum_{\substack{A \subset \intr{K}:\\ i \in A}} \mathbb{E}^{(1)}[N_A] \text{KL}\left(\mathbb{P}_A^{(1)}, \mathbb{P}_A^{(i)}\right) \right\rbrace \label{eq:tv_tsy2}.
	\end{align}
	Now fix $A \subset \intr{K}$, such that $i \in A$. Let us calculate $\text{KL}(\mathbb{P}_A^{(1)}, \mathbb{P}_A^{(i)})$. 
	
	Denote $\mathbb{Q}_A^{(i)}$ for $i \in \intr{K}$, the joint distribution: $((X_{j,t})_{j \in A},U_t)$. Using the data processing inequality
	\begin{equation*}
	\text{KL}\left(\mathbb{P}_A^{(1)}, \mathbb{P}_A^{(i)}\right) \le \text{KL}\left(\mathbb{Q}_{A}^{(1)}, \mathbb{Q}_{A}^{(i)}\right).
	\end{equation*}
	Observe that $X_1$ follows the same distribution under $\mathbb{P}^{(1)}$ and $\mathbb{P}^{(i)}$. Therefore
	\begin{equation*}
	\text{KL}\left(\mathbb{P}_A^{(1)}, \mathbb{P}_A^{(i)}\right) \le \mathbb{E}^{(1)}_{U_t}\left[ \text{KL}\left(\mathbb{P}_A^{(1)}\left((X_{j})_{j \in A} \mid U_t\right), \mathbb{P}_A^{(i)}\left((X_{j})_{j \in A} \mid U_t\right)\right) \right],
	\end{equation*}
	where $\mathbb{E}_{U_t}^{(.)}$ denotes the conditional expectation under $\mathbb{P}^{(.)}$ with respect to $U_t$.

	\noindent Observe that conditionally to $U_t$, the variables $(X_j)_{j \neq 1}$ are independent, both under $\mathbb{P}_{A}^{(1)}$ and $\mathbb{P}_{A}^{(1)}$. Moreover $X_j$ for $j \in A \setminus \{1,i\}$ have the same probability distributions under both alternatives. Therefore
	
	\begin{align*}
	\text{KL}\left(\mathbb{P}_A^{(1)}, \mathbb{P}_A^{(i)}\right) &\le \mathbb{E}^{(1)}_{U_{t}} \left[\text{KL}\left(\mathbb{P}_A^{(1)}\left((X_{j})_{j\in A} \mid U_{t}\right), \mathbb{P}_A^{(i)}\left((X_{j})_{j\in A} \mid U_{t}\right) \right) \right]\\
	&= \mathbb{E}^{(1)}_{U_{t}} \left[\text{KL}\left(\mathbb{P}_A^{(1)}\left(X_{i} \mid U_{t}\right), \mathbb{P}_A^{(i)}\left(X_{i} \mid U_{t}\right) \right) \right].
	\end{align*}
	
	Therefore, we have:
	\begin{align*}
	\text{KL}\left(\mathbb{P}_A^{(1)}, \mathbb{P}_A^{(i)}\right) &\le (1-a_i) \text{KL}\left( \mathds{1}\left(W_{i,t} \le b_i \right), \mathds{1}\left(W'_{i,t} \le \frac{\mu_1+\epsilon-\mu_i}{1-a_i} + b_i\right)\right).
	\end{align*}
	Using Lemma~\ref{lem:ber2},  we have
	\begin{align*}
	\text{KL}\left(\mathbb{P}_A^{(1)}, \mathbb{P}_A^{(i)}\right) &\le (1-a_i) \frac{\left(\frac{\mu_1+\epsilon-\mu_i}{1-a_i}\right)^2}{\left(\frac{\mu_1+\epsilon-\mu_i}{1-a_i} + b_i \right) \left(1-\frac{\mu_1+\epsilon-\mu_i}{1-a_i} - b_i\right)}\\
	&= \frac{(1-a_i) (\mu_1+\epsilon-\mu_i)^2}{(\mu_1+\epsilon - \mu_i + b_i -a_ib_i)(1-(\mu_1 + \epsilon-\mu_i) -a_i-b_i+a_ib_i)}. 
	\end{align*}
	Using $a_i+b_i-a_ib_i = \mu_i$ and $a_i = \frac{\mu_i-b_i}{1-b_i}$, we have
	\begin{align*}
	\text{KL}\left(\mathbb{P}_A^{(1)}, \mathbb{P}_A^{(i)}\right) &\le \frac{(1-\mu_i)(\mu_1+\epsilon -\mu_i)^2}{\left((\mu_1+\epsilon)(1-b_i)-\mu_i +b_i\right)(1-\mu_1-\epsilon)}\\
	&= \frac{(1-\mu_i)(\mu_1+\epsilon -\mu_i)^2}{\left(\mu_1+\epsilon-\mu_i + b_i(1-\mu_1-\epsilon)\right)(1-\mu_1-\epsilon)} 
	\end{align*}

	Next, we plug the inequality above into inequality \eqref{eq:tv_tsy2} and obtain:
	\begin{equation*}
	\text{TV}\left(\mathbb{P}^{(1)}, \mathbb{P}^{(i)}\right) \le 1 - \frac{1}{2} \exp\left\lbrace - \text{KL}\left(\mathbb{P}_A^{(1)}, \mathbb{P}_A^{(i)}\right) \sum_{\substack{A \subset \intr{K}:\\ i \in A}} \mathbb{E}^{(1)}[N_A] \right\rbrace.
	\end{equation*}
	For $j \in \intr{K}$, denote by $T_j$ the total number of rounds where arm $j$ was queried: 
	\[
	T_j := \sum_{\substack{A \subset \intr{K}:\\ j \in A}} N_A.
	\]
	Therefore
	
	\begin{equation*}
	\text{TV}\left(\mathbb{P}^{(1)}, \mathbb{P}^{(i)}\right) \le 1 - \frac{1}{2} \exp\left\lbrace - \text{KL}\left(\mathbb{P}_A^{(1)}, \mathbb{P}_A^{(i)}\right) \mathbb{E}^{(1)}[T_i] \right\rbrace.
	\end{equation*}
	Combining the inequality above with \eqref{eq:bound_tv1} and taking $\epsilon \to 0$, we obtain
	\begin{equation*}
	\mathbb{E}^{(1)}\left[T_i\right] \ge \frac{\left(\mu_1+\epsilon-\mu_i + b_i(1-\mu_1-\epsilon)\right)(1-\mu_1-\epsilon)}{(1-\mu_i)(\mu_1+\epsilon-\mu_i)^2} \log\left(\frac{1}{4\delta}\right).
	\end{equation*}
	Taking $\epsilon \to 0$, using the expression of $b_i$ and $\mu_i,\mu_1 \le 3/4$ yields
	%	\begin{align}
	%		\mathbb{E}^{(1)}\left[T_i\right] &\ge \frac{\mu_1-\mu_i +  v^2_{i1} - (\mu_1-\mu_i) + (\mu_1-\mu_i)^2}{3(\mu_1-\mu_i)^2} \log\left(\frac{1}{4\delta}\right) \notag\\
	%		&\ge \frac{\mu_1-\mu_i +  v^2_{i1} - (\mu_1-\mu_i) + (\mu_1-\mu_i)^2}{3(\mu_1-\mu_i)^2} \log\left(\frac{1}{4\delta}\right) \notag\\
	%		&\ge \frac{v^2_{i1} + (\mu_1-\mu_i)}{6(\mu_1 - \mu_i)^2}\log\left(\frac{1}{4\delta}\right).\label{eq:ti1}
	%	\end{align}
	\begin{align}
	\mathbb{E}^{(1)}\left[T_i\right] &\ge \frac{(\mu_1-\mu_i) +  V_{i1}  + (\mu_1-\mu_i)^2}{2(\mu_1-\mu_i)^2}\cdot \frac{1-\mu_1}{1-\mu_i} \log\left(\frac{1}{4\delta}\right) \notag\\
	&\ge \frac{V_{i1} + (\mu_1-\mu_i)}{8(\mu_1 - \mu_i)^2}\log\left(\frac{1}{4\delta}\right).\label{eq:ti1}
	\end{align}

	\paragraph{Developing a lower bound for $T_i$ in the case $V_{i1} \ge \mu_1+\mu_i - \mu_1^2 - \mu_i^2$.} In this case we introduce the following alternative distribution $\mathbb{P}^{(i)}$, where we only change the variable $X_i$ into:
	$\forall t \ge 1: X_{i,t} = \mathds{1}\left( W_{i,t} \le \mu_1+\epsilon \right)$, for some small positive constant $\epsilon$. 
	
	Similarly to the previous case, since all variables are independent: for any $A \subset \intr{K}$:
	
	Therefore
	\begin{align*}
	\text{KL}\left(\mathbb{P}_A^{(1)}, \mathbb{P}_A^{(i)}\right) &= \text{KL}\left(\mathds{1}(W_{i,t} \le \mu_i), \mathds{1}( W_{i,t} \le \mu_1+\epsilon)\right)\\
	&\le \frac{(\mu_1+\epsilon-\mu_i)^2}{(\mu_1+\epsilon)(1-\mu_1-\epsilon)}\\	
	\end{align*}
	The remainder of the calculation is similar to the preivous case and leads to (using $\mu_i \in [1/3, 3/4]$):
	\begin{align}
	\mathbb{E}^{(1)}[T_i] &\ge \frac{\mu_1(1-\mu_1)}{(\mu_1 - \mu_i)^2}\log\left(\frac{1}{4\delta}\right)\notag\\
	&\ge \frac{3}{16(\mu_1 - \mu_i)^2}\log\left(\frac{1}{4\delta}\right) \notag\\
	&\ge \frac{3(V_{i1}+\mu_1-\mu_i)}{16(\mu_1 - \mu_i)^2}\log\left(\frac{1}{4\delta}\right)\label{eq:ti2},
	\end{align}
	where we used in the last line the fact that: $V_{i1}+\mu_1-\mu_i\leq 5/4$ 
	
	\paragraph{Conclusion}
	Using \eqref{eq:ti1} and \eqref{eq:ti2}, we have for any $i\in \{2, \dots, K\}$:
	\begin{equation*}
	\mathbb{E}^{(1)}[T_i] \ge \frac{V_{i1}+(\mu_1-\mu_i)}{8(\mu_1-\mu_i)^2} \log\left(\frac{1}{4\delta}\right).
	\end{equation*}
	Recall that $T_i$ represents the number of rounds where arm $i$ is queried. Hence the total number of queries $N$ satisfies: $N \ge \sum_{i=2}^{K} T_i$. We conclude that 
	\begin{equation*}
	\mathbb{E}^{(1)}\left[N \right] \ge \log\left(\frac{1}{4\delta}\right) \sum_{i=2}^{K} \frac{V_{i1}+(\mu_1-\mu_i)}{8(\mu_1-\mu_i)^2}. 
	\end{equation*}

	\section{Proof of Theorem~\ref{thm:lower_bound_g}}
	
	Without loss of generality assume that $\mu_1 > \mu_2 \ge \dots \ge \mu_K$, hence $i^* = 1$. Let $\mathcal{A}$ be a $\delta$ sound strategy. Let $\mathbb{P}^{(1)}$ denote the joint distribution of the arms defined as follows: Let $Z$ denote a random variable distributed following $\mathcal{N}(\mu_1, 1)$.
	\begin{itemize}
		\item For $i \in \{2, \dots, K\}$, we set $X_i = Z + Y_i$, where $Y_i$ is a random variable independent of $Z$ and $(Y_i)_{i \ge 2}$, and distributed following $\mathcal{N}(\mu_i - \mu_1, V_{i1})$.
		\item The optimal arm is given by $X_1 = Z$, let $\sigma_1 := 1$.  
	\end{itemize}
	Recall that in the configuration above we have for each $i \in \intr{K}$: $\mathbb{E}[X_i] = \mu_i$, $\mathbb{E}[(X_1 - X_i)^2] = V_{i1}$, and $\text{Var}(X_i) \ge 1$, therefore $\mathbb{P}^{(0)} \in \mathbb{G}_K(\bm{\mu}, \bm{V})$. 
	
	For $i \in \{2, \dots, K\}$, denote by $\mathbb{P}^{(i)}$ the alternative probability distribution where the only arm modified is arm $i$. We set: $X_i = Z+Y'$, where $Y'\sim \mathcal{N}(\epsilon, V_{i1}+(\mu_1-\mu_i)^2)$, where $\epsilon >0$.
	
	Observe that, under $\mathbb{P}^{(i)}$, the arm $i$ is optimal. Fix $i \ge 2$. Since $\mathcal{A}$ is $\delta$-sound, we have 
	\begin{equation}\label{eq:delta_sound}
	\mathbb{P}^{(1)} (\psi \neq 1) + \mathbb{P}^{(i)}(\psi \neq i) \le 2\delta.
	\end{equation}

	\noindent Using Theorem 2.2 of \cite{tsybakov2003optimal}, we have
	\[
	\mathbb{P}^{(1)}(\psi \neq 1) + \mathbb{P}^{(i)}(\psi = 1) \ge 1 - \text{TV}\left(\mathbb{P}^{(1)}, \mathbb{P}^{(i)} \right),
	\]
	where $\text{TV}(\mathbb{P}, \mathbb{Q})$ denotes the total variation distance between $\mathbb{P}$ and $\mathbb{Q}$. Using \eqref{eq:delta_sound}, we deduce that
	\begin{equation}\label{eq:bound_tv}
	\text{TV}(\mathbb{P}^{(1)}, \mathbb{P}^{(i)}) \ge 1-2\delta. 
	\end{equation}
	
	\noindent For a subset $A \subset \intr{K}$, denote $N_A$ the total number of rounds where the jointly queried arms are the elements of $A$, and let $\mathbb{P}_A^{(.)}$ denote the joint distribution of $(X_j)_{j \in A}$ under $\mathbb{P}^{(.)}$. Under Protocol~\ref{algo:gp}, the learner has to choose a subset $C_t \subset \intr{K}$ in each round $t$, and observes only the rewards of arms in $C_t$. Hence, we can apply Lemma~\ref{lem:tvkl} which bounds the total variation distance in terms of Kullback-Leibler discrepancy to this case where arms correspond to subsets of $\intr{K}$. This gives
	
	\begin{align}
	\text{TV}\left(\mathbb{P}^{(1)}, \mathbb{P}^{(i)}\right) &\le 1 - \frac{1}{2} \exp\left\lbrace -\sum_{A \subset \intr{K}} \mathbb{E}^{(1)}[N_A] \text{KL}\left(\mathbb{P}_A^{(1)}, \mathbb{P}_A^{(i)}\right) \right\rbrace \nonumber\\
	&= 1 - \frac{1}{2} \exp\left\lbrace -\sum_{\substack{A \subset \intr{K}:\\ i \in A}} \mathbb{E}^{(1)}[N_A] \text{KL}\left(\mathbb{P}_A^{(1)}, \mathbb{P}_A^{(i)}\right) \right\rbrace \label{eq:tv_tsy}.
	\end{align}
	Now fix $A \subset \intr{K}$, such that $i \in A$. Let us calculate $\text{KL}(\mathbb{P}_A^{(1)}, \mathbb{P}_A^{(i)})$. 
	
	Using the data processing inequality, we deduce that, for any $A \subset \intr{K}$,
	\begin{equation*}
	\text{KL}\left(\mathbb{P}_A^{(1)}, \mathbb{P}_A^{(i)}\right) \le \text{KL}\left(\mathbb{P}_{A\cup \{1\}}^{(1)}, \mathbb{P}_{A\cup \{1\}}^{(i)}\right).
	\end{equation*}
	Observe that $X_1$ follows the same distribution under $\mathbb{P}^{(1)}$ and $\mathbb{P}^{(i)}$. Therefore
	\begin{equation*}
	\text{KL}\left(\mathbb{P}_A^{(1)}, \mathbb{P}_A^{(i)}\right) \le \mathbb{E}^{(1)}_{X_1}\left[ \text{KL}\left(\mathbb{P}_A^{(1)}\left((X_{j})_{j \in A} \mid X_1\right), \mathbb{P}_A^{(i)}\left((X_{j})_{j \in A} \mid X_1\right)\right) \right], 
	\end{equation*}
	where $\mathbb{E}_{X_1}[.]$ refers to the conditional expectation with respect to $X_1$. Recall that conditionally to $X_1$, the variables $(X_j)_{j \neq 1}$ are independent, both under $\mathbb{P}_{A}^{(1)}$ and $\mathbb{P}_{A}^{(1)}$. Moreover $X_j$ for $j \in A \setminus \{1,i\}$ have the same probability distributions under both alternatives. Therefore, we can the above Kullback divergence in terms of conditional Kullback divergence.
	\begin{align*}
	\text{KL}\left(\mathbb{P}_A^{(1)}, \mathbb{P}_A^{(i)}\right) &\le \mathbb{E}^{(1)}_{X_{1}} \left[\text{KL}\left(\mathbb{P}_A^{(1)}\left((X_{j})_{j\in A} \mid X_{1}\right), \mathbb{P}_A^{(i)}\left((X_{j})_{j\in A} \mid X_{1}\right) \right) \right]\\
	&= \mathbb{E}^{(1)}_{X_{1}} \left[\text{KL}\left(\mathbb{P}_A^{(1)}\left(X_{i} \mid X_{1}\right), \mathbb{P}_A^{(i)}\left(X_{i} \mid X_{1}\right) \right) \right]\ .
	\end{align*}
	Moreover, we have
	\begin{align*}
	\mathbb{E}_A^{(1)}\left[X_i \mid X_{1}\right] &= \mu_i + X_{1} - \mu_{1}\\
	\mathbb{E}_A^{(i)}\left[X_i \mid X_{1}\right] &= \mu_1 + \epsilon + X_{1} - \mu_{1}\\
	\Var_A^{(1)}\left( X_i \mid X_{1} \right) &= 1+V_{i1} \\ 
	\Var_A^{(i)}\left( X_i \mid X_{1} \right) &= 1+V_{i1}+(\mu_1-\mu_i)^2\ .
	\end{align*}
	We deduce that
	\begin{align*}
	\text{KL}\left(\mathbb{P}_A^{(1)}, \mathbb{P}_A^{(i)}\right) &\le \frac{1}{2}\log\left(\frac{1+V_{i1}+(\mu_1-\mu_i)^2}{1+V_{i1}}\right) + \frac{1+V_{i1}+(\mu_i - \mu_1-\epsilon)^2}{2(1+V_{i1})} - \frac{1}{2}\\
	&\le \frac{1}{2}\log\left( \frac{ V_{i1}+(\mu_1-\mu_i)^2}{ V_{i1} } \right) + \frac{(\mu_i - \mu_1-\epsilon)^2-(\mu_i - \mu_1)^2}{2(1+V_{i1}+(\mu_1-\mu_i)^2)}\ .
	\end{align*}
	As a conclusion, we have 
	\begin{equation}\label{eq:kl_01}
	\text{KL}\left(\mathbb{P}_A^{(1)}, \mathbb{P}_A^{(i)}\right) \le \frac{1}{2}\log\left( \frac{ V_{i1}+(\mu_1-\mu_i)^2}{ V_{i1} } \right) + \frac{(\mu_i - \mu_1-\epsilon)^2-(\mu_i - \mu_1)^2}{2(1+V_{i1}+(\mu_1-\mu_i)^2)}\  .
	\end{equation}

	Next, we plug the inequality above into \eqref{eq:tv_tsy} and obtain:
	\begin{equation*}
	\text{TV}\left(\mathbb{P}^{(1)}, \mathbb{P}^{(i)}\right) \le 1 - \frac{1}{2} \exp\left\lbrace - \text{KL}\left(\mathbb{P}_A^{(1)}, \mathbb{P}_A^{(i)}\right) \sum_{\substack{A \subset \intr{K}:\\ i \in A}} \mathbb{E}^{(1)}[N_A] \right\rbrace.
	\end{equation*}
	For $j \in \intr{K}$, denote by $T_j$ the total number of rounds where arm $j$ was queried: 
	\[
	T_j := \sum_{\substack{A \subset \intr{K}:\\ j \in A}} N_A.
	\]
	Therefore
	\begin{equation*}
	\text{TV}\left(\mathbb{P}^{(1)}, \mathbb{P}^{(i)}\right) \le 1 - \frac{1}{2} \exp\left\lbrace - \text{KL}\left(\mathbb{P}_A^{(1)}, \mathbb{P}_A^{(i)}\right) \mathbb{E}^{(1)}[T_i] \right\rbrace.
	\end{equation*}
	Combining the inequality above with \eqref{eq:bound_tv} and taking $\epsilon \to 0$, we obtain
	\begin{equation*}
	\mathbb{E}^{(1)}\left[T_i\right] \ge \frac{2}{\log\left(\frac{V_{i1}+(\mu_1-\mu_i)^2}{V_{i1}} \right)} \log\left(\frac{1}{4\delta}\right).
	\end{equation*}
	Taking $\epsilon \to 0$ and using the definition of $\sigma_i$ for $i \ge 2$, we have
	\begin{equation*}
	\mathbb{E}^{(1)}\left[\sum_{i=2}^{K} T_i \right] \ge 2\log\left(\frac{1}{4\delta}\right) \sum_{i=2}^{K} \frac{1}{\log\left(\frac{V_{i1}+(\mu_1-\mu_i)^2}{V_{i1}} \right)}. 
	\end{equation*}
	Recall that $T_i$ represents the number of rounds where arm $i$ is queried. Hence the total number of queries $N$ satisfies: $N \ge \sum_{i=2}^{K} T_i$. We conclude that 
	\begin{equation*}
	\mathbb{E}^{(1)}\left[N \right] \ge 2\log\left(\frac{1}{4\delta}\right) \sum_{i=2}^{K} \frac{1}{\log\left(1 + \frac{(\mu_1-\mu_i)^2}{V_{i1}}\right)} . 
	\end{equation*}

	\section{Some technical results}
	
	We state below a version of the empirical Bernstein's inequality presented in \cite{audibert2007tuning}.
	
	\begin{theorem}\label{conc:bern}
		Let $X_1, \dots, X_t$ be i.i.d random variables taking their values in $[0,b]$. Let $\mu = \mathbb{E}[X_1]$ be their common expected value. Consider the empirical expectation $\bar{X}_t$ and variance $V_t$ defined respectively by
		\[
		\bar{X}_t = \frac{\sum_{i=1}^{t}X_i}{t} \quad \text{ and } \quad V_t = \frac{\sum_{i=1}^{t}(X_i - \bar{X}_t)^2}{t}.
		\]
		Then for any $t \in \mathbb{N}$ and $x>0$, with probability at least $1-3e^{-x}$
		\[
		\lvert \bar{X}_t - \mu \rvert \le \sqrt{\frac{2V_tx}{t}} + \frac{3bx}{t}.
		\]
	\end{theorem}
	
	Theorem below corresponds to Theorem 10 in \cite{maurer2009empirical}.
	
	\begin{theorem}\label{thm:concvar}
		Let $n\ge 2$ and $\bm{X}=(X_1, \dots, X_n)$ be a vector of independent random variables with values in $[0,1]$. Then for $\delta>0$ we have, writing $\mathbb{E}V_n$ for $\mathbb{E}_{\bm{X}}V_n(\bm{X})$,
		\begin{align*}
		\mathbb{P}\left\lbrace \sqrt{\mathbb{E}V_n} > \sqrt{V_n(\bm{X})} + \sqrt{\frac{2\log(1/\delta)}{n-1}} \right\rbrace &\le \delta\\
		\mathbb{P}\left\lbrace \sqrt{V_n(\bm{X})} > \sqrt{\mathbb{E}V_n} + \sqrt{\frac{2\log(1/\delta)}{n-1}} \right\rbrace &\le \delta.
		\end{align*}
		
	\end{theorem}

	The following lemma is technical, it will be used in the proof of Lemma~\ref{lem:ultram}.
	
	\begin{lemma}\label{cl:calpure_2}
		Let $a,b,c \text{ and } d >0$, we have
		\[
		\frac{a+b}{c+d} \le \max\left\lbrace \frac{a}{c}, \frac{b}{d} \right\rbrace.
		\]
	\end{lemma}
	\begin{proof}
		Let $\rho = \frac{c}{c+d} \in (0,1)$. Observe that
		\[
		\frac{a+b}{c+d} = \rho \frac{a}{c} + (1-\rho) \frac{b}{d},
		\]
		and $1-\rho = \frac{d}{c+d} \in (0,1)$. Taking the maximum of the convex combination above gives the result. 
	\end{proof}

	\begin{lemma}\label{cl:calpure2}
		Let $x \ge 1, c \in (0,1)$ and $y >0$ such that:
		\begin{equation}
		\label{eq:hypineq}
		\frac{\log(x/c)}{x} > y.
		\end{equation}
		Then:
		\begin{equation*}
		x < \frac{ 2\log\left( \frac{1}{cy} \right) }{ y}.
		\end{equation*}
	\end{lemma}
	
	\begin{proof}
		%Since $x/c >x \geq 3$, it holds $\log(x/c) \geq 1$.
		Inequality~\eqref{eq:hypineq}
		implies
		\[
		x < \frac{\log(x/c)}{y},
		\]
		and further
		\[
		\log(x/c) < \log(1/yc) + \log \log(x/c) \leq \log(1/yc) + \frac{1}{2}\log(x/c),            
		\]
		since it can be easily checked that $\log(t) \leq t/2$ for all $t>0$.
		Solving and plugging back into the previous display leads to the claim.
	\end{proof}
	
	\begin{lemma}[\citealp{kaufmann2016complexity}, with slight modification]\label{lem:tvkl}
		Let $\nu$ and $\nu'$ be two collections of $d$ probability distributions on $\R$, such that for all $a\in \intr{d}$, the distributions $\nu_a$ and $\nu_{a'}$ are mutually absolutely continuous. For any almost-surely finite stopping time $\tau$ with respect to $(\mathcal{F}_t)$,
		\[	
		\sup_{\mathcal{E} \in \mathcal{F}_{\tau}} \lvert \mathbb{P}_{\nu}\left(\mathcal{E}\right) - \mathbb{P}_{\nu'}\left(\mathcal{E}\right) \rvert \le 1-\frac{1}{2} \exp\left\lbrace-\sum_{a=1}^{d} \mathbb{E}_{\nu}\left[N_{a}(\tau)\right] \text{KL}(\nu_a, \nu'_a)\right\rbrace.
		\] 
	\end{lemma}

	\begin{lemma}\label{lem:chi22}\cite{verzelen2012minimax}
		Let $(Y_1, \dots, Y_n)$ be i.i.d Gaussian variables, with mean $0$ and variance $1$. Let $Z = \sum_{i=1}^{n} Y_i^2$. For any number $0<x<1$,
		\begin{align*}
		\mathbb{P}\left(Z \ge n-1 + 2\sqrt{(n-1)\log(1/x)}+2\log(1/x) \right) &\le x,\\
		\mathbb{P}\left( Z \le n-1 - 2 \sqrt{(n-1)\log(1/\delta)}\right) &\le x.
		\end{align*}
		For any positive number $0<x<1$
		\begin{equation*}
		\mathbb{P}\left(Z \le (n-1)Cx^{2/(n-1)}\right) \le x,
		\end{equation*}
		where the constant $C=\exp(-1)$.
	\end{lemma}
	%The following lemma develops a concentration bound for the Gaussian chaos of order two. This result is taken from Example 2.12 of \cite{boucheron2013concentration}.
	
	\paragraph{Concentration bound for the Gaussian variance sample}
	Let $\bm{X} = (X_1, \dots, X_n)$ be a vector of independent standard normal variables. Define the sample variance by 
	\begin{equation}\label{eq:def_sample_var}
	V_n(\bm{X}) := \frac{1}{n(n-1)} \sum_{1\le i<j \le n} (X_i - X_j)^2. 
	\end{equation}
	
	Observe that
	\begin{align*}
	V_n(\bm{X}) &= \frac{1}{2n(n-1)} \sum_{i \neq j} (X_i - X_j)^2\\
	&= \frac{1}{n(n-1)} \left((n-1) \sum_{i=1}^{n}X_i^2 - \sum_{i\neq j} X_i X_j\right)\\
	&= \frac{1}{n(n-1)} \bm{X}^{\top} A \bm{X},
	\end{align*}
	where $A$ is the matrix such that off-diagonal entries are equal to $-1$ and diagonal entries are equal to $n-1$. 
	Let us compute the eigenvalue of the matrix $A$: Observe that $A = n \bm{I}_n - \mathds{1}_n\mathds{1}_n^{\top}$, hence the eigenvalues of $A$ are $n$ with multiplicity $n-1$ and $0$. Hence, we have 
	\begin{equation*}
	\bm{X}^{\top} A \bm{X} = n\sum_{i=1}^{n-1}Y_i^2,
	\end{equation*}
	where $(Y_i)$ are independent and follow the standard normal distribution. Finally using Lemma~\ref{lem:chi22} we obtain:
	\begin{lemma}\label{lem:conc_var_g}
		Let $\bm{X} = (X_1, \dots, X_n)$ be a vector of independent standard normal variables. Let $V_n(\bm{X})$ denote the variance sample defined in \eqref{eq:def_sample_var}. Let $\delta \in (0,1/3)$, we have with probability at least $1-3\delta$:
		\begin{align*}
		V_n(\bm{X}) &\ge \max\left\lbrace 1 - 2 \sqrt{\frac{\log(1/\delta)}{n-1}}; C \delta^{2/(n-1)} \right\rbrace\\
		V_n(\bm{X}) &\le 1 + 2 \sqrt{\frac{\log(1/\delta)}{n-1}} + 2 \frac{\log(1/\delta)}{n-1}.
		\end{align*}
	\end{lemma}
	
	\begin{lemma}\label{lem:ber}
		Let $X$ and $Y$ be two Bernoulli variables with means $x$ and $y$ respectively. We have
		\[
		\max\left(x+y-1, 0\right) \le \mathbb{E}[XY] \le \min(x,y).
		\]
		Moreover:
		\begin{equation*}
		x-y - (x - y)^2 \le \Var(X - Y) \le \min\left(2-(x+y);~x+y \right)-(x-y)^2.
		\end{equation*}
	\end{lemma}
	\begin{proof}
		Without loss of generality suppose that $x\le y$. We have
		\begin{align*}
		\mathbb{E}\left[XY\right] &= \mathbb{P}(XY=1)\\
		&= \mathbb{P}(X=1 \quad \text{and} \quad Y=1)\\
		&= \mathbb{P}(X=1)+\mathbb{P}(Y=1) - \mathbb{P}(X=1 \quad \text{or} \quad Y=1)\\
		&= x+y - \mathbb{P}(X=1 \quad \text{or} \quad Y=1).
		\end{align*}
		The conclusion follows by using $y \le \mathbb{P}(X=1 \quad \text{or} \quad Y=1) \le 1$, and $XY\ge 0$.
		
		\noindent Moreover, we have:
		\begin{align*}
		\Var(X - Y) &= \mathbb{E}\left[(X - Y)^2\right] - (x - y)^2\\
		&= \mathbb{E}[X] + \mathbb{E}[Y] -2 \mathbb{E}[XY] - (x - y)^2\\
		&= x+y - (x - y)^2 -2\mathbb{E}[XY].
		\end{align*}
		We plug-in the previous bounds on $\mathbb{E}[XY]$ and obtain the result.
	\end{proof}

	\begin{lemma}\label{lem:ber2}
		Let $X$ and $Y$ denote two Bernoulli variables with paramters $x \in (0,1)$ and $y \in (0,1)$ respectively. We have
		\[
		\text{KL}(X,Y) \le \frac{(x-y)^2}{y(1-y)}.
		\]
	\end{lemma}
	\begin{proof}
		The proof is a direct consequence of the bound
		\[
		\text{KL}(X,Y) \le \frac{x^2}{y}+\frac{(1-x)^2}{1-y}-1.
		\]
	\end{proof}
%\vskip 0.2in
%\bibliographystyle{plain}
%\bibliography{oomp}
%\input{oomp.bbl}
\end{document}